\documentclass[10pt,journal,compsoc]{IEEEtran}
\ifCLASSINFOpdf
\else
\fi
\usepackage{amsmath}
\usepackage{amsthm}
\usepackage{graphicx}
\usepackage{subfigure}
\usepackage{booktabs}
\usepackage{colortbl}
\usepackage{xcolor}
\usepackage{multirow}
\usepackage{enumitem}
\usepackage{tikz}
\usetikzlibrary{graphs, positioning,  shapes.geometric}

\usepackage{booktabs} 
\usepackage{color}
\usepackage{chemformula}

\usepackage[ruled,linesnumbered]{algorithm2e}
\usepackage{hyperref}
\usepackage{epstopdf}
\usepackage{amsmath,amssymb,mathrsfs}
\usepackage{multirow}

\newtheorem{theorem}{Theorem}

\newtheorem{lemma}{Lemma}

\newtheorem{definition}{Definition}
\newtheorem{assumption}{Assumption}
\newtheorem{remark}{Remark}

\DeclareMathOperator{\topk}{top} 
\DeclareMathOperator{\blurs}{blurs} 
\DeclareMathOperator{\blur}{blur} 
 
\DeclareMathOperator{\randk}{rand}

\begin{document}
%
\title{Towards the Flatter Landscape and Better Generalization in Federated Learning under Client-level Differential Privacy
}
%
%
%

\author{
Yifan Shi, Kang Wei,~\IEEEmembership{Member,~IEEE}, Li Shen, Yingqi Liu,  Xueqian Wang,~\IEEEmembership{Member,~IEEE},\\
Bo Yuan,~\IEEEmembership{Senior Member,~IEEE} and Dacheng Tao,~\IEEEmembership{Fellow,~IEEE}
        
\IEEEcompsocitemizethanks{
\IEEEcompsocthanksitem An earlier version of this paper was presented in part at the 2023 IEEE/CVF Conference on Computer Vision and Pattern Recognition (CVPR) \cite{shi2023make}. 

\IEEEcompsocthanksitem Yifan Shi and Xueqian Wang are with the Shenzhen International Graduate School, Tsinghua University, Shenzhen 518055, China (e-mail: shiyf21@mails.tsinghua.edu.cn; wang.xq@sz.tsinghua.edu.cn). 

\IEEEcompsocthanksitem Kang Wei and Yingqi Liu are with the School of Electrical and Optical Engineering, Nanjing University of Science and Technology, Nanjing 210094, China (e-mail:\{kang.wei, lyq\}@njust.edu.cn).  

\IEEEcompsocthanksitem Li Shen is with the JD Explore Academy, Beijing, China (e-mail: mathshenli@gmail.com).

\IEEEcompsocthanksitem Bo Yuan is with Shenzhen Wisdom and Strategy Technology Co., Ltd., Shenzhen 518055, China (e-mail: boyuan@ieee.org).

\IEEEcompsocthanksitem Dacheng Tao is with the University of Sydney, NSW 2006, Australia. (e-mail: dacheng.tao@gmail.com).

}
}

%
%


\markboth{Journal of \LaTeX\ Class Files,~Vol.~18, No.~9, September~2020}%
{Flatter Landscape and Higher Generalization against Local Update Perturbation \\in Federated Learning}
%



\IEEEtitleabstractindextext{
\begin{abstract}
To defend the inference attacks and mitigate the sensitive information leakages in Federated Learning (FL), client-level Differentially Private FL (DPFL) is the de-facto standard for privacy protection by clipping local updates and adding random noise. However, existing DPFL methods tend to make a sharp loss landscape and have poor weight perturbation robustness, resulting in severe performance degradation. To alleviate these issues, we propose a novel DPFL algorithm named DP-FedSAM, which leverages gradient perturbation to mitigate the negative impact of DP. Specifically, DP-FedSAM integrates Sharpness Aware Minimization (SAM) optimizer to generate local flatness models with improved stability and weight perturbation robustness, which results in the small norm of local updates and robustness to DP noise, thereby improving the performance. To further reduce the magnitude of random noise while achieving better performance, we propose DP-FedSAM-$\topk_k$ by adopting the local update sparsification technique.
From the theoretical perspective, we present the convergence analysis to investigate how our algorithms mitigate the performance degradation induced by DP. Meanwhile, we give rigorous privacy guarantees with Rényi DP, the sensitivity analysis of local updates, and generalization analysis. At last, we empirically confirm that our algorithms achieve state-of-the-art (SOTA) performance compared with existing SOTA baselines in DPFL.  
\end{abstract}

}

\maketitle

\section{Introduction}
\IEEEPARstart{F}ederated Learning (FL) \cite{Li2020federated} allows distributed clients to collaboratively train a shared model without sharing data. 
However, FL faces the severe dilemma of privacy leakage \cite{Kairouz2021Advances}.
Recent works show that a curious server can infer clients’ privacy information such as membership and data features, by well-designed generative models and/or shadow models~\cite{Fredrikson2015model,Shokri2017membership,Melis2018inference,Nasr2019comprehensive,zhang2022fine}. To address this issue, differential privacy (DP) \cite{Dwork2014the} has been introduced in FL, which can protect every instance in any client's data (instance-level DP~\cite{Agarwal2018cpsgd,Hu2021federated, Sun2021federated,Sun2021practical}) or the information between clients (client-level DP~\cite{McMahan2018learning,Robin2017differentially,KairouzL2021the,wei2021user, Rui2022Federated,Anda2022differentially}). In general, client-level DP is more suitable to apply in the real-world setting due to better model performance. For instance, a language prediction model with client-level DP \cite{geyer2017differentially, McMahan2018learning} has been applied on mobile devices by Google.
In general, the Gaussian noise perturbation-based method is commonly adopted for ensuring strong client-level DP.
However, this method includes two operations: clipping the $l_2$ norm of local updates to a sensitivity threshold $C$ and adding random noise proportional to the model size, whose standard deviation (STD) is also decided by $C$. These steps may cause severe performance degradation \cite{cheng2022differentially, hu2022federated}, especially on large-scale complex model \cite{simonyan2014very}, such as ResNet-18 \cite{he2016deep}, or with heterogeneous data.  

\begin{figure*}
\centering
\begin{minipage}[b]{0.55 \linewidth}
    \centering
    \includegraphics[width=1.0\linewidth]{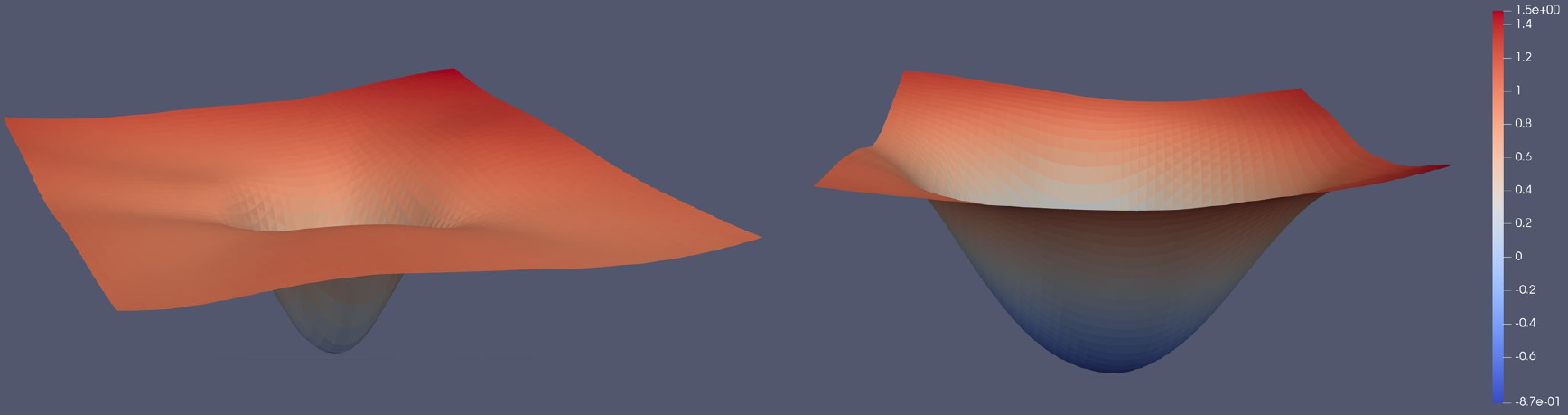}
    \centerline{(a) Loss landscapes.}\medskip
\end{minipage}
\hfill
\begin{minipage}[b]{0.42\linewidth}
    \centering
    \includegraphics[width=1.0\linewidth]{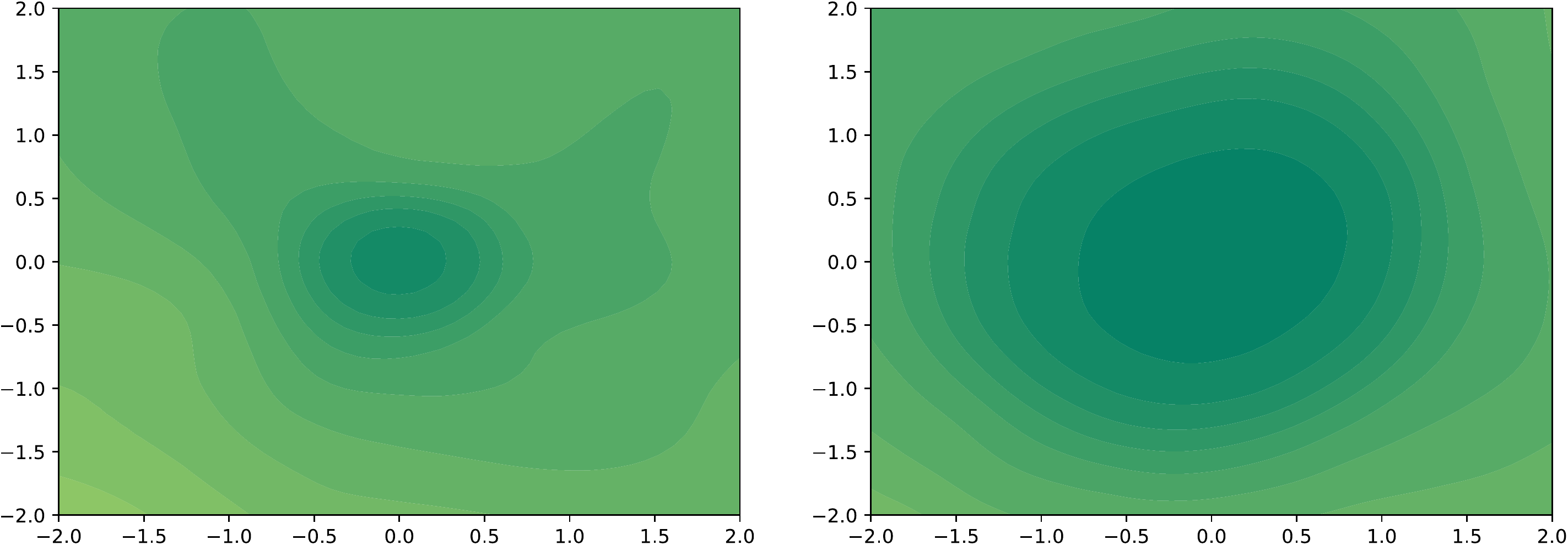}
    \centerline{(b) Loss surface contours.}\medskip
    \label{contour_fedavg_dpfedavg}
\end{minipage}
\vspace{-0.35cm}
\caption{\small Loss landscapes (a) and surface contours (b) comparison between DP-FedAvg (left) and FedAvg (right).
}
\vspace{-0.4cm}
\label{landscape_fedavg_dpfedavg}
\end{figure*}

The reasons behind this issue are two-fold: (i) The useful information is dropped due to the clipping operation, especially with small $C$ values, which is contained in the local updates; (ii) The model inconsistency among local models is exacerbated as the addition of random noise severely damages local updates and leads to large variances between local models, especially with large $C$ values \cite{cheng2022differentially}. 
Existing works try to overcome these issues via restricting the norm of local update \cite{cheng2022differentially} and leveraging local update sparsification technique \cite{cheng2022differentially, hu2022federated} to reduce the adverse impacts of clipping and adding random noise. However, the model performance degradation is still significant compared with FL methods without considering privacy, such as FedAvg \cite{mcmahan2017communication}. 

\subsection{Motivation} To further explore this phenomenon, we compare the structure of loss landscapes and surface contours \cite{li2018visualizing} of FedAvg \cite{mcmahan2017communication} and DP-FedAvg \cite{McMahan2018learning,Robin2017differentially} on the partitioned CIFAR-10 dataset \cite{krizhevsky2009learning} with Dirichlet distribution ($\alpha =0.6$) and ResNet-18 backbone \cite{he2016deep} in Figure \ref{landscape_fedavg_dpfedavg} (a) and (b), respectively. Note that the convergence of DP-FedAvg is worse than FedAvg as its loss value is higher after a long communication round. Furthermore, FedAvg features a flatter landscape, whereas DP-FedAvg has a sharper one, resulting in both poorer generalization ability (sharper minima, see Figure \ref{landscape_fedavg_dpfedavg} (a)) and weight perturbation robustness (see Figure \ref{landscape_fedavg_dpfedavg} (b)), which is caused by the clipped local update information and exacerbated model inconsistency, respectively. Based on these observations, an interesting research question is: \emph{can we further alleviate the performance degradation by making the landscape flatter and generalization better? }

To answer this question, we propose DP-FedSAM with gradient perturbation to improve model performance. Specifically, a local flat model is generated by the SAM optimizer \cite{foret2021sharpnessaware} in each client, which leads to improved stability. After that, a potentially global flat model can be generated by aggregating several flat local models, which results in better generalization ability and higher robustness to DP noise, thereby significantly improving the performance and achieving a better trade-off between performance and privacy. To further reduce the magnitude of random noise while achieving better performance, we propose DP-FedSAM-$\topk_k$ based on DP-FedSAM by adopting the local update sparsification technique \cite{hu2022federated,cheng2022differentially}. 
Theoretically, we present a tighter bound  
$\small \mathcal{O}(\frac{1}{\sqrt{KT}}+\frac{\sum_{t=1}^T(\overline{\alpha}^t\sigma_g^2 + \tilde{\alpha}^t  L^2)}{T^2} + \frac{L^2 \sqrt{T}\sigma^2C^2pd}{m^2\sqrt{K}} )$ in the stochastic non-convex setting, where both $\frac{1}{T}\sum_{t=1}^T\overline{\alpha}^t$ and $\frac{1}{T}\sum_{t=1}^T\tilde{\alpha}^t$ are bounded constants. $p$, $K$, and $T$ are sparsity ratio, local iteration steps, and communication rounds, respectively. 
Specifically, the on-average norm of local updates $\overline{\alpha}^{t}$ and local update consistency among clients $\tilde{\alpha}^t$ before clipping and adding noise operations are represented as
\begin{equation} \small
\begin{split}
    \overline{\alpha}^{t} :=\frac{1}{M} \sum_{i=1}^{M} \alpha^t_i~~~ \text{and}  ~~~ \tilde{\alpha}^t :=\frac{1}{M}\sum_{i=1}^{M} |\alpha_i^t - \overline{\alpha_i^t}|, \nonumber
\end{split}
\end{equation}
where $\alpha^t_i = \min (1, \frac{C}{ \eta  \|  \sum_{k=0}^{K-1}  \tilde{\mathbf{g}}^{t,k}(i) \|_2} ) $ measures the negative impact of local update clipping at $t$-th communication round for client $i$. Note that $\eta$ is the learning rate and  $\tilde{\mathbf{g}}^{t,k}(i) $ is the gradient using the local SAM optimizer at the $k$-th local iteration. 
Next, we present how SAM mitigates the ill impact of DP. 
For the clipping operation, DP-FedSAM reduces the $l_2$ norm and the negative impact of the inconsistency among local updates on convergence. For adding noise operation, we obtain higher weight perturbation robustness for reducing the performance damage caused by random noise, thereby being more robust to DP noise in the FL training process. Meanwhile, we deliver sensitivity, privacy, sensitivity, and generalization analysis for our algorithms.
Empirically, we conduct extensive experiments on EMNIST, CIFAR-10,
and CIFAR-100 datasets in both the independently and identically distributed (IID) and Non-IID settings. Furthermore, we investigate the loss landscapes, surface contours, and the norm distribution of local updates for exploring the intrinsic effects of SAM with DP, which together with the theoretical analysis confirms the effect of DP-FedSAM.

\subsection{Contributions} The main contributions of our work are four-fold.
\begin{itemize}
    \item We propose two novel schemes DP-FedSAM and DP-FedSAM-$\topk_k$ in DPFL to alleviate the performance degradation issue (see Section \ref{methods}).
    \item We establish an improved convergence rate, which is tighter than the conventional bounds \cite{cheng2022differentially,hu2022federated} in the stochastic non-convex setting. Moreover, we provide rigorous privacy guarantees, sensitivity,  and generalization analysis (see Section \ref{th}).
    \item We are the first to in-depth analyze the roles of the on-average norm of local updates $\overline{\alpha}^{t}$ and local update consistency among clients $\tilde{\alpha}^t$ on convergence (see Theorem \ref{th:conver} of Section \ref{sth:conver}). Meanwhile, we empirically validate the theoretical results for mitigating the adverse impacts of the norm of local updates  (see Section \ref{exper_DP}).
    \item We conduct extensive experiments to verify the effect of our algorithms, which can achieve state-of-the-art (SOTA) performance compared with several strong DPFL baselines (see Section \ref{exper}).
\end{itemize}


\subsection{Organizations}
Section 2 reviews the related work on client-level DPFL, SAM optimizer, and model compression in DPFL. Section 3 introduces the background of FL, DP, and local update sparsification. The proposed DP-FedSAM and DP-FedSAM-$\topk_k$ are described in Section 4.
Moreover, we present the theoretical analysis of sensitivity, convergence, and generalization for our methods in Section 5.
Extensive experimental evaluation is presented in Section 6. This paper is concluded in Section 7 with suggested directions for future work.


\vspace{-0.15cm}

\section{Related work}
\noindent
\textbf{Client-level DPFL.}
Client-level DPFL is the de-facto solution for protecting each client's data. DP-FedAvg \cite{dpfedavg} is the first attempt in this setting, which trains a language prediction model in a mobile keyboard and ensures client-level DP guarantee by employing the Gaussian mechanism and composing privacy guarantees.
After that, the work in \cite{Kairouz2021TheDD,Thakkar2019adaclip} presents a comprehensive end-to-end system, which appropriately discretizes the data and adds discrete Gaussian noise before performing secure aggregation. Meanwhile, AE-DPFL \cite{Zhu2020VotingbasedAF} leverages the voting-based mechanism among the data labels instead of averaging the gradients to avoid dimension dependence
and significantly reduce the communication cost. Fed-SMP \cite{hu2022federated} uses Sparsified Model Perturbation
(SMP) to mitigate the impact of privacy protection on model accuracy. Different from the aforementioned methods, a recent study \cite{cheng2022differentially} revisits this issue and 
leverages Bounded Local Update Regularization (BLUR) and Local Update Sparsification (LUS) to restrict the norm of local updates and reduce the noise size before executing operations that guarantee DP, respectively. Nevertheless, the issue of performance degradation still remains. 
%



\noindent
\textbf{Sharpness Aware Minimization (SAM).} 
SAM \cite{foret2021sharpnessaware} is an effective optimizer for training deep learning (DL) models, which leverages the flatness geometry of the loss landscape to improve model generalization ability. Recently, the work in\cite{Andriushchenko2022Towards} investigates the properties of SAM and provides convergence results of SAM for non-convex objectives. As a powerful optimizer, SAM and its variants have been applied to various machine learning (ML) tasks \cite{Zhao2022Penalizing,kwon2021asam,du2021efficient,liu2022towards,Abbas2022Sharp-MAML,mi2022make,shi2023improving,zhong2022improving, huangrobust,sunfedspeed,sun2023adasam}. Specifically, the studies in \cite{Qu2022Generalized}, \cite{sunfedspeed}, and \cite{Caldarola2022Improving} integrate SAM to improve the generalization, and thus mitigate the distribution shift problem and achieve a new SOTA performance for FL. However, to the best of our knowledge, only limited efforts have been devoted to the empirical performance and theoretical analysis of SAM in DPFL \cite{shi2023make}. 
In this paper, we extend the work in \cite{shi2023make} towards a more detailed and systematic evaluation of SAM local optimizer with local update sparsification.


\noindent
\textbf{Sparsification in DPFL.} To achieve a better trade-off between performance and privacy protection, many works leverage the sparsification technique in privacy protection to introduce a large amount of random noise \cite{hu2022federated, cheng2022differentially, Hu2021federated,zhu2021improving,liu2020fedsel}. It retains only the relatively large weights of each layer of the local model with a sparsity ratio of $k/d$ ($d$ is the weight scale) while the rest weights are set to zero. The advantage is that the amount of random noise can be reduced (no noise needs to be added to the sparse weight positions), and the performance can be improved. In DPFL, the sparsification technique can be divided into two strategies: random sparsification and weight-based sparsification. For instance, Fed-SPA~\cite{Hu2021federated} integrates random sparsification with gradient perturbation to obtain a better utility-privacy trade-off and reduce the communication cost for instance-level DPFL. Fed-SMP \cite{hu2022federated} uses Sparsified Model Perturbation (SMP) with these two strategies to mitigate the impact of privacy protection on model accuracy. The work in \cite{cheng2022differentially} leverages weight-based sparsification to reduce the noise size before executing operations that guarantee DP. 



The most related works to this paper are DP-FedAvg \cite{McMahan2018learning}, Fed-SMP \cite{hu2022federated}, and DP-FedAvg with LUS and BLUR \cite{cheng2022differentially}. However, these works suffer from inferior performance due to the exacerbated model inconsistency among the clients caused by random noise. Different from existing works, we try to alleviate this issue by making the landscape flatter and weight perturbation ability more robust. Furthermore, another related work is FedSAM \cite{Qu2022Generalized}, which integrates the SAM optimizer to enhance the flatness of the local model and achieves new SOTA performance for FL. On top of the aforementioned studies, we are the first to extend the SAM optimizer into DPFL to effectively alleviate the performance degradation issue. Meanwhile, we simultaneously provide the theoretical analysis for sensitivity, privacy, convergence, and generalization in the non-convex setting. Finally, we empirically verify our theoretical results and the performance superiority compared with existing SOTA methods in DPFL.

\section{Preliminary}

In this section, we first introduce the problem setup of FL and then introduce the key terminologies in DP. Finally, we present the sparsification technique.

\subsection{Federated Learning}
Consider a general FL system consisting of $M$ clients where each client owns its local dataset.
Let $S_i$ denote the training sample set held by client $i$, respectively, where $i\in \mathcal{U} = \{1, 2,\ldots, M\}$.
Formally, the FL task is expressed as:
\begin{equation}
\small
\mathbf{w}^{\star} = \mathop{\arg\min}_{\mathbf{w}}\sum_{i\in \mathcal{U}}p_{i}f_i(\mathbf{w},   S_{i}),
\end{equation}
where $p_{i} = \vert S_i\vert/\vert   S\vert\geq 0$ with $\sum_{i\in \mathcal{U}}{p_{i}}=1$, and $f_i(\cdot)$ is the local loss function with $f_i(\mathbf{w}) = F_i(\mathbf{w}; \xi_i)$, $\xi_i$ is a batch sample data in client $i$. 
We assume $S = \{(x_1, y_1), \ldots, (x_N, y_N) | x_i \in \mathcal X \subset \mathbb R^{d_X},$ $ y_i \in \mathcal Y \subset \mathbb R^{d_Y}, i = 1, \ldots, N\}$ is the whole training dataset held by all clients, where $x_i$ is the $i$-th feature and $y_i$ is the corresponding label; $d_X$ and $d_Y$ are the dimensions of the feature and the label, respectively. Meanwhile, we define $z_i = (x_i, y_i)$ and assume $z_i$ satisfies the data distribution $\mathcal D$. $\vert S_{i}\vert $ is the size of training dataset $S_i$ and $\vert   S\vert = \sum_{i\in \mathcal{U}}{\vert   S_{i}\vert}$ is the total size of training datasets.
For the $i$-th client, a local model is learned on its private training data $ S_i$ by:
\begin{equation}
\small
\mathbf{w}_{i}=\mathbf{w}_{i}^{t}-\eta \nabla f_i(\mathbf{w}_{i},   S_{i}).
\end{equation}
Generally, the local loss function $f_i(\cdot)$ has the same expression across each client.
Then, the $M$ associated clients collaboratively learn a global model $\mathbf{w}$ over the heterogeneous training data $  S_{i}$, $\forall i \in \mathcal{U}$.

\subsection{Differential Privacy}
Differential Privacy (DP) \cite{Dwork2014the} is a rigorous privacy notion for measuring privacy risk.  In this paper, we consider a relaxed version: R\'{e}nyi DP (RDP)~\cite{mironov2017renyi} for privacy calculation in Theorem \ref{th:privacy}. Furthermore, we adopt the max divergence for delivering the generalization bound in Theorem \ref{th:gener_bound}.
\begin{definition}[R\'{e}nyi DP, \cite{mironov2017renyi}] Given a real number $\alpha \in (1, \infty )$ and privacy parameter $\rho \ge 0$, a randomized mechanism $\mathcal{M}$ satisfies $(\alpha, \rho)$-RDP if for any two adjacent datasets $U$, $U'$ that differ in a single sample data, the R\'{e}nyi $\alpha$-divergence between $\mathcal{M}(U)$ and $\mathcal{M}(U^{\prime})$ satisfies:
\begin{equation}
\small
\!\!\!D_{\alpha}\left[\mathcal{M}(U) \| \mathcal{M}\left(U^{\prime}\right)\right]\!:=\!\frac{1}{\alpha\!-\!1} \log \mathbb{E}\left[\left(\frac{\mathcal{M}(U)}{\mathcal{M}\left(U^{\prime}\right)}\right)^{\alpha}\!\right] \!\leq \!\rho,
\end{equation}
where the expectation is taken over the output of $\mathcal{M}(U^{\prime})$.
\end{definition}
R\'{e}nyi DP is a useful analytical tool to measure privacy and accurately represent guarantees on the tails of the privacy loss, which is strictly stronger than $(\epsilon, \delta)$-DP for $\delta > 0 $. We provide the privacy analysis based on this tool for each user's privacy loss. 
Furthermore, we deliver high-probability generalization bound by proving that the FL training process satisfies the $(\epsilon, \delta)$-DP with the $\delta$-approximate max divergence during the communication rounds following \cite{he2021tighter}.
\begin{definition}[$\delta$-approximate max divergence, \cite{Dwork2014the}] 
\label{def:max_diver}
For any random variables $X\in U$ and $Y\in U$, where $U$ is a dataset, the $\delta$-approximate divergence between $X$ and $Y$ is defined as
    \begin{equation}
    \small
        D^\delta_\infty(X\parallel Y) \!= \! \max_{U \subseteq  \mathrm{Supp}(X): \mathbb{P}(Y\in U) \ge \delta} \Bigg[\log \frac{\mathbb{P}(Y\in U) - \delta}{\mathbb{P}(Y\in U)}\Big].
    \end{equation}
\end{definition}

\begin{definition}
[$l_2$ Sensitivity, \cite{cheng2022differentially}]
\label{sensitivity}
Let $\mathcal{F}$ be a function, the $L_{2}$-sensitivity of $\mathcal{F}$ is defined as $\mathcal{S}=\max _{U \simeq U^{\prime} } \| \mathcal{F}\left(U\right)-$ $\mathcal{F}\left(U^{\prime}\right) \|_{2}$, where the maximization is taken over all pairs of adjacent datasets $U$ and $U'$.
\end{definition}

The sensitivity of a function $\mathcal{F}$ captures the magnitude by which an individual’s data can change the function $\mathcal{F}$ in the worst case.
Therefore, it plays a crucial role in determining the magnitude of noise required to ensure DP.

\begin{definition}[Client-level DP, \cite{mcmahan2017learning}] 
 A randomized algorithm $\mathcal{M}$ is $(\epsilon, \delta)$-DP if for any two adjacent datasets $U$, $U^{\prime}$ constructed by adding or removing all data of any client, every possible subset of outputs $O$ satisfies the following inequality:
\begin{equation}
\operatorname{Pr}[\mathcal{M}(U) \in O] \leq e^{\epsilon} \operatorname{Pr}\left[\mathcal{M}\left(U^{\prime}\right) \in O\right]+\delta.
\end{equation}
\end{definition}

In client-level DP, we aim to ensure participation information for any clients. Therefore, we need to make local updates similar whether one client participates or not.

\subsection{Local Update Sparsification}
To reduce the amount of random noise and achieve a better trade-off between performance and privacy protection, existing works usually adopt the local update sparsification technique in client-level DPFL, also called top-$k$ sparsifier based on the magnitude of weight in each layer. 
\begin{definition}[Local update sparsification, \cite{cheng2022differentially,hu2022federated}]
\label{def:sparsifier}
For $1\leq k\leq d$ and the local update vector $\mathbf{x}\in\mathbb{R}^d$, the top-$k$ sparsifier $\topk_k:\mathbb{R}^d \rightarrow \mathbb{R}^d$ is defined as
\begin{align}
&[\topk_k(\mathbf{x})]_j:=
\begin{cases} 
[\mathbf{x}]_{\pi(j)}, & \text{if } j\leq k\\ 
0, & \text{otherwise}
\end{cases},
\end{align}
where $\Omega_k = {[d] \choose k}$ denotes the set of all $k$-element subsets of $[d]$ and $\pi$ is a permutation of $[d]$ such that $|[\mathbf{x}]_{\pi(j)}| \geq |[\mathbf{x}]_{\pi(j+1)}|$ for $j\in[1,d-1]$. Note that the sparsity ratio $p = k/d$ after performing the local update sparsification.
\end{definition}
Therefore, after clipping local model updates and adding random noise in DP, the top-$k$ sparsifier discards the local update parameters that are unlikely to be important, and only the $k$ parameters/coordinates with the largest magnitude are transmitted to the server.

\section{Methodology}\label{methods}

To revisit the performance degradation challenge in DPFL, we investigate the loss landscapes and surface contours of FedAvg and DP-FedAvg in Figure \ref{landscape_fedavg_dpfedavg} (a) and (b), respectively. We find that the DPFL method produces a sharper landscape with both poorer generalization ability and weight perturbation robustness than the FL method. It means that the DPFL method may result in poor flatness and make the model sensitive to noise perturbation. In this paper, we plan to approach this challenge from the optimizer perspective by adopting a SAM optimizer in each client, dubbed DP-FedSAM, whose local loss function is defined as:
\begin{equation} \small\label{Eq:sam}
\small
    f_i(\mathbf{w}) = \mathbb{E}_{\xi\sim \mathcal{D}_i}\max_{\|\delta_i\|_2 \leq \rho} F_i(\mathbf{w}^{t,k}(i) +\delta_i; \xi_i), \quad i \in \mathcal{N},
\end{equation}
where $\mathbf{w}^{t,k}(i) +\delta_i$ is the perturbed model and $\rho$ is a predefined constant controlling the radius of the perturbation while $\|\cdot\|_2$ is the $l_2$-norm.
Instead of searching for a solution via SGD \cite{bottou2010large,bottou2018optimization}, SAM \cite{foret2021sharpnessaware} aims to seek a solution in a flat region by adding a small perturbation, i.e., $w + \delta$ with more robust performance. 
Specifically, for each client $i\in\{1,2,...,M\}$ and each local iteration $k \in \{0,1,...,K-1\}$ in each communication round $t \in \{0,1,...,T-1\}$, the $k$-th inner iteration in client $i$ is performed as:
\begin{gather}
\small
        \mathbf{w}^{t,k+1}(i)=\mathbf{w} ^{t,k}(i)-\eta \tilde{\mathbf{g}}^{t,k}(i),  \label{local iteration} \\
       \!\!\!\! \tilde{\mathbf{g}}^{t,k}(i)=\nabla F_i(\mathbf{w}^{t,k} + \delta(\mathbf{w}^{t,k});\xi),  \delta(\mathbf{w}^{t,k})=\frac{\rho \mathbf{g}^{t,k}}{\left \| \mathbf{g}^{t,k} \right \|_2}, \label{g} \!\!\!
\end{gather}
where $\delta(\mathbf{w}^{t,k})$ is calculated by the first-order Taylor expansion around $\mathbf{w}^{t,k}$ \cite{foret2021sharpnessaware}.
After that,
we adopt a sampling mechanism, gradient clipping, and add Gaussian noise to ensure client-level DP. Note that this sampling method can amplify the privacy guarantee since it decreases the chances of leaking information about a particular individual. 
After $m$ clients are sampled with probability $q=m/M$ at each communication round, which is important for measuring privacy loss. 
We clip the local updates of these $m$ sampled clients as:
\begin{equation} \small
\small
\tilde{\Delta}^t_{i} = \Delta_{i}^{t}\cdot\min\left(1, \frac{C}{\Vert \Delta_{i}^{t}\Vert_2}\right). 
\end{equation}
After clipping, we add Gaussian noise to the local update to ensure client-level DP as follows:
\begin{equation}\label{dp_c}
\small
\breve{\Delta}^t_{i} = \tilde{\Delta}^t_{i} + \mathcal{N}(0, \sigma^2C^2 \cdot \mathbf{I}_d/m).
\end{equation}
With a given noise variance $\sigma^2$, the accumulative privacy budget $\epsilon$ can be calculated based on the sampled Gaussian mechanism~\cite{Yousefpour2021Opacus}. 

After that, we adopt the sparsification technique according to Definition \ref{def:sparsifier} to reduce the magnitude of added random noise, named DP-FedSAM-$\topk_k$. Where the local model update retains the largest magnitude of all parameters according to the sparsity ratio $p$. Formally, under Definition \ref{def:sparsifier}, the operation is defined as:
\begin{equation}\small
\hat{\Delta}^t_{i} = \breve{\Delta}^t_{i} \odot \boldsymbol{m}^t ,
\label{eq:weight-based_mask}
\end{equation}
where the operator $\odot$ represents the element-wise multiplication and the $j$-th coordinate of mask vector $\boldsymbol{m}^t$ equals to $1$ if it is selected by the top-$k$ sparsifier and DP-FedSAM-$\topk_k$ is the same as DP-FedSAM when $p=1.0$.  
We summarize the training procedure in Algorithm \ref{DFedSAM_DP}.

\begin{algorithm}[ht]
\small
\caption{DP-FedSAM and DP-FedSAM-$\topk_k$}
\label{DFedSAM_DP}
\SetKwData{Left}{left}\SetKwData{This}{this}\SetKwData{Up}{up} \SetKwFunction{Union}{Union}\SetKwFunction{FindCompress}{FindCompress}
\SetKwInOut{Input}{Input}\SetKwInOut{Output}{Output}
\Input{Total number of clients $M$, sampling ratio of clients $q$, total number of communication rounds $T$, the clipping threshold $C$, local learning rate $\eta$, and total number of the local iterates are $K$.} 
\Output{Global model $\mathbf{w}^{T}$.}
\textbf{Initialization:} Randomly initialize the global model $\mathbf{w}^{0}$.\\
\For{$t=0$ \KwTo $T-1$}{
    Sample a set of $m=qM $ clients at random without replacement, denoted by $\mathcal{W}^t $.
    
    \For{client $i=1$ \KwTo $m$ \emph{\textbf{in parallel}} }{
        \For{$k=0$ \KwTo $K-1$ }{
        Update the global parameter as local parameter
        $\mathbf{w}^{t}(i) \gets \mathbf{w}^{t}$.
        
        Sample a batch of local data $\xi_i$ and calculate local gradient $\mathbf{g}^{t,k}(i)=\nabla F_i(\mathbf{w}^{t,k}(i);\xi_i)$.
        
        Gradient perturbation by Equation \eqref{g}.
        
        Local iteration update by Equation \eqref{local iteration}.
        }
        $ \Delta ^t_i = \mathbf{w}^{t,K}(i) - \mathbf{w}^{t,0}(i)$.
        
        Clip and add noise for DP by Equation \eqref{dp_c}.

        Generate local update $\hat{\Delta} ^t_i$ to be uploaded to the server side via \color{gray}{Option I} \color{black}{or} \color{gray}{II}.\\
        \color{black}{\textbf{Return} $\hat{\Delta} ^t_i$.}
    }
    $\mathbf{w}^{t+1} \gets \mathbf{w}^{t} + \frac{1}{m}\sum_{i\in \mathcal{W}^t} \hat{\Delta} ^t_i $.
}
\color{gray}{Option I: (DP-FedSAM)}\\
\color{black}{Get $\hat{\Delta} ^t_i$ without sparsification: $\hat{\Delta}^t_{i} = \breve{\Delta}^t_{i}$.}

\color{gray}{Option II: (DP-FedSAM-$\topk_k$)}\\
\color{black}{Get $\hat{\Delta} ^t_i$ with local update sparsification by Equation \eqref{eq:weight-based_mask}.}\\
\end{algorithm}
Compared with existing DPFL methods \cite{McMahan2018learning,hu2022federated,cheng2022differentially},
the benefits of our algorithms lie in four-fold: (i) We introduce SAM into DPFL to alleviate severe performance degradation via seeking a flat model in each client, which is caused by the exacerbated inconsistency of local models. Specifically, DP-FedSAM features both better generalization ability and robustness to DP noise by making the loss landscape of the global model flatter;
(ii) We analyze in detail how DP-FedSAM mitigates the negative impacts of DP. We theoretically analyze the convergence with the on-average norm of local updates $\overline{\alpha}^{t}$ and local update consistency among clients $\tilde{\alpha}^t$, and empirically confirm these results via observing the norm distribution and average norm of local updates (see Section \ref{exper_DP}); (iii) We deliver the sensitivity, privacy, and generalization analysis for our algorithms (see Section \ref{th});
(iv)
We also present the theories unifying the impacts of gradient perturbation $\rho$ in SAM, the on-average norm of local updates $\overline{\alpha}^{t}$, and local update consistency among clients $\tilde{\alpha}^t$ in clipping operation, and the variance of random noise $\sigma^2C^2/m$ upon the convergence rate (see Section \ref{sth:conver}); (v) We explore local update sparsification with the SAM local optimizer to further achieve performance improvement.
\section{Theoretical Analysis}\label{th}

In this section, we give a rigorous analysis of our algorithms, including the sensitivity, privacy, and convergence rate. The detailed proof is presented in \textbf{Appendix} \ref{appendix_th}. Below, we first give several key assumptions.

\begin{assumption}[Lipschitz smoothness] \label{a1}
 The function $F_i$ is differentiable and $\nabla F_i$ is $L$-Lipschitz continuous, $\forall i \in \{1,2,\ldots,M\}$, i.e.,
$\|\nabla F_i({\bf w}) - \nabla F_i({\bf w'})\| \leq L \|{\bf w} - {\bf w'}\|,$
for ${\bf w}, {\bf w'} \in \mathbb{R}^d$.
\end{assumption}

\begin{assumption}[Bounded variance] \label{a2} The gradient of the function $f_i$ have $\sigma_l$-bounded variance, i.e.,
$\mathbb{E}_{\xi_i}\left\|\nabla F_i (\mathbf{w}^k(i);\xi_i ) -\nabla F_i (\mathbf{w}(i))\right  \|^2 \leq \sigma_l^2$, 
$\forall i \in \{1,2,\ldots,M\}, k \in \{1, ..., K-1\}$, and
the global variance is also bounded, i.e., $\frac{1}{M} \sum_{i=1}^M \|\nabla f_i({\bf w}) - \nabla f({\bf w})\|^2 \leq \sigma_{g}^2$ for ${\bf w} \in \mathbb{R}^d$. It is not hard to verify that the $\sigma_g$ is smaller than the homogeneity parameter $\beta$, i.e., $\sigma_g^2 \leq \beta^2$.
\end{assumption}

\begin{assumption}[Bounded gradient]\label{a3}
 For any $i \!\in\! \{1,2,\ldots,M\}$ and ${\bf w}\!\in\! \mathbb{R}^d$, we have $\|\nabla f_i({\bf w})\|\!\leq\! B $.
\end{assumption}


Note that the above assumptions are mild and commonly used in characterizing the convergence rate of FL \cite{Sun2022Decentralized,ghadimi2013stochastic,yang2021achieving,bottou2018optimization,reddi2020adaptive, Qu2022Generalized, cheng2022differentially, hu2022federated}. Furthermore, gradient clipping in Deep Learning is often used to prevent the gradient explosion phenomenon, and thereby the gradient is bounded. 
The technical challenge of our algorithms lies in: (i) how SAM mitigates the impact of DP; (ii) how to analyze in detail the impacts of the consistency among clients and on-average norm of local updates caused by clipping operation.  


\subsection{Sensitivity Analysis}
First, we study the sensitivity of local update  $\Delta^t_i$ from client $i \in \{1,2,..., M\}$ before clipping at the $t$-th communication round. This upper bound of sensitivity can roughly measure the degree of privacy protection.
Under Definition \ref{sensitivity}, the sensitivity can be denoted by $\mathcal{S}_{\Delta_i^t}$ in client $i$ at the $t$-th communication round. 
\begin{theorem}[Sensitivity analysis]\label{th:sensitivity}
Denote $\Delta_i^t(\bf x)$ and $\Delta_i^t(\bf y)$ as the local updates at the $t$-th communication round and the models $\mathbf{x}(i)$ and $\mathbf{y}(i)$ are conducted on two datasets which differ at only a single sample. Assuming the initial model parameter $\mathbf{w}^t(i)=\mathbf{x}^{t, 0}(i) =\mathbf{y}^{t, 0}(i)$, the expected squared  
sensitivity $\mathcal{S}^2_{\Delta_i^t}$ of local update is upper bounded by
\begin{align}
\small
    \mathbb{E}\mathcal{S}^2_{\Delta_i^t} \leq
     \frac{6\eta^2\rho^2KL^2(12K^2L^2\eta^2+ 10)}{1-2\eta^2L^2 K}
\end{align}
When the local adaptive learning rate satisfies $\eta=\mathcal{O}({1}/{L\sqrt{KT}})$ and the perturbation amplitude $\rho$ is
proportional to the learning rate, e.g., $\rho = \mathcal{O}(\frac{1}{\sqrt{T}})$, we have
\begin{align}
\small
    \mathbb{E}\mathcal{S}^2_{\Delta_i^t} \leq
    \mathcal{O}\left(\frac{1}{T^2}\right). 
\end{align}
The expected squared sensitivity of local update with SGD in DPFL is $ \mathbb{E}\mathcal{S}^2_{\Delta_i^t, SGD} \leq \frac{6\eta^2\sigma_l^2K}{1-3\eta^2KL^2}$.
Thus $\mathbb{E}\mathcal{S}^2_{\Delta_i^t, SGD} \leq \mathcal{O}(\frac{\sigma_l^2}{KL^2T})$ when $\eta=\mathcal{O}({1}/{L\sqrt{KT}})$.
\begin{remark}
It is clear that the upper bound of $  \mathbb{E}\mathcal{S}^2_{\Delta_i^t, SAM}$ is tighter than that of $\mathbb{E}\mathcal{S}^2_{\Delta_i^t, SGD}$. For privacy protection, it implies that DP-FedSAM has a better privacy guarantee than DP-FedAvg.
Meanwhile, in local iteration, our algorithms feature both better model consistency among clients and training stability.
\end{remark}
\end{theorem}


\subsection{Privacy Analysis}
To achieve client-level privacy protection, we analyze the sensitivity of the aggregation process after clipping the local updates. Below, we present the privacy analysis.
\begin{lemma}
The sensitivity of client-level DP in DP-FedSAM can be expressed as $C/m$.
\end{lemma}
\begin{proof}
Given two adjacent batches $\mathcal{W}^{t}$ and $\mathcal{W}^{t,\rm{adj}}$, where $\mathcal{W}^{t,\rm{adj}}$ contains one extra or less client, we have
\begin{equation} \small
\small
\left\Vert \frac{1}{m}\sum_{i\in \mathcal{W}^{t}}\Delta^{t}_{i}-\frac{1}{m}\sum_{j\in \mathcal{W}^{t,\rm{adj}}} \Delta^{t}_{j} \right\Vert_{2} = \frac{1}{m}\left\Vert \Delta^{t}_{j'} \right\Vert_{2}\leq \frac{C}{m},
\end{equation}
where $\Delta^{t}_{j'}$ is the local update of the client where the two batches differ.
\end{proof}
\begin{remark}
The value of sensitivity can determine the amount of variance for adding random noise.
\end{remark}
Existing work \cite{hu2022federated} has shown that SGD and sparsification satisfy the R\'{e}nyi DP and the SAM optimizer only adds perturbation on the basis of SGD and affects the model during training. Since both SAM and sparsification are performed before the DP process, they all satisfy the R\'{e}nyi DP. Therefore, 
after adding the Gaussian noise, we calculate the accumulative privacy budget~\cite{Yousefpour2021Opacus} along with training as follows using R\'{e}nyi DP. 
\begin{theorem}[Privacy calculation] \label{th:privacy}
After $T$ communication rounds, the accumulative privacy budget is calculated by:
\begin{equation} \small\label{eq:accumulative_eps}
\begin{aligned}
\epsilon = \overline{\epsilon} + \frac{(\alpha-1)\log(1-\frac{1}{\alpha})-\log(\alpha)-\log(\delta)}{\alpha-1},
\end{aligned}
\end{equation}
where
\begin{equation} \small
\begin{aligned}
\overline{\epsilon} &= \frac{T}{\alpha-1}\ln {\mathbb{E}_{z\sim \mu_{0}(z)}\left[\left(1-q+\frac{q \mu_{1}(z)}{\mu_{0}(z)}\right)^{\alpha}\right]},
\end{aligned}
\end{equation}
and $q$ is the sampling rate for client selection; $\mu_{0}(z)$ and $\mu_{1}(z)$ denote the Gaussian probability density function (PDF) of $\mathcal{N}(0,\sigma)$ and the mixture of two Gaussian distributions $q\mathcal{N}(1,\sigma)+(1-q)\mathcal{N}(0,\sigma)$, respectively; $\sigma$ is the noise STD in Eq. \eqref{dp_c}; $\alpha$ is a tunable variable.
\end{theorem}
\begin{remark}
A small sampling rate $q$ can enhance the privacy guarantee by decreasing the privacy budget, but it may also degrade the training performance due to the number of participating clients being reduced in each communication round. 
\end{remark}

\subsection{Convergence Analysis}\label{sth:conver}
Below, we give a convergence analysis of how DP-FedSAM mitigates the negative impacts of DP. The key contribution is that we jointly consider the impacts of the on-average norm of local updates $\overline{\alpha}^{t}$ and the local update consistency among clients $\tilde{\alpha}^t$ on the rate. Moreover, we also empirically validate these results in Section \ref{exper_DP}.

\begin{theorem}[Convergence bound]\label{th:conver}
Under assumptions 1-4, the local learning rate satisfies $\eta=\mathcal{O}({1}/{L\sqrt{KT}})$ and let $f^{*}$ denotes the minimal value of $f$, i.e., $f(x)\ge f(x^*)=f^*$ for all $x\in \mathbb{R}^{d}$. Given the sparsity ratio $p$,
the perturbation amplitude $\rho$ 
proportional to the learning rate, $\rho = \mathcal{O}(1/\sqrt{T})$, and
the sequence of outputs $\{\mathbf{w}^t\}$ generated by Alg. \ref{DFedSAM_DP}, we have:
\begin{equation*}
\small
\begin{split}
\small
    & \frac{1}{T} \sum_{t=1}^T
    \mathbb{E}\left[\overline{\alpha}^{t}\left\|\nabla f\left(\mathbf{w}^{t}\right)\right\|^{2}\right]   \leq  
    \underbrace{\mathcal{O}\left(\frac{2L(f({\bf w}^{1})-f^{*})}{\sqrt{KT}} + \frac{ L^2\sigma_{l}^2}{KT^2}\right)}_{\text{From FedSAM}} \\ 
    &\qquad\quad +
    \underbrace{
     \underbrace{\mathcal{O}\left( \frac{\sum_{t=1}^T(\overline{\alpha}^t  \sigma_{g}^2 + \tilde{\alpha}^t  L^2  )}{T^2}  \right)}_{\text{Clipping}}
    + \underbrace{ \mathcal{O}\left(\frac{L^2 \sqrt{T}\sigma^2C^2pd}{m^2\sqrt{K}} \right)}_{\text{Adding noise}}
    }_{\text{From operations for DP}} 
    \end{split}
\end{equation*}
where
\begin{equation} \small
\begin{split}
    \overline{\alpha}^{t} :=\frac{1}{M} \sum_{i=1}^{M} \alpha^t_i~~~ \text{and}  ~~~ \tilde{\alpha}^t :=\frac{1}{M}\sum_{i=1}^{M} |\alpha_i^t - \overline{\alpha_i^t}|, 
\end{split}
\end{equation}
with $\alpha^t_i = \min (1, \frac{C}{ \eta  \|  \sum_{k=0}^{K-1}  \tilde{\mathbf{g}}^{t,k}(i) \|_2} ) $. Note that $\overline{\alpha}^{t}$ and $\tilde{\alpha}^t$ measure the on-average norm of local updates and local update consistency among clients before clipping and adding noise operations in DP-FedSAM, respectively. 
\end{theorem}
\begin{table*}
\caption{Averaged training accuracy (\%) and testing accuracy (\%) on two data in both IID and Non-IID settings for all compared methods.}
\vspace{-0.2cm}
\centering
\renewcommand{\arraystretch}{0.7}
\label{table:all_baselines}
\resizebox{1\linewidth}{!}{
\begin{tabular}{cccccccc} 
\toprule
\multirow{2}{*}{\textbf{Task}} & \multirow{2}{*}{\textbf{Algorithm}} & \multicolumn{2}{c}{Dirichlet~0.3} & \multicolumn{2}{c}{Dirichlet~0.6} & \multicolumn{2}{c}{IID}  \\ 
\cmidrule{3-8}
                      &                            & Train         & Validation          & Train         & Validation          & Train      & Validation           \\ 
\midrule 
       & DP-FedAvg       & 99.28$\pm$0.02 & 73.10$\pm$0.16          & 99.55$\pm$0.02 & 82.20$\pm$0.35          & 99.66$\pm$0.40 & 81.90$\pm$0.86           \\
       & Fed-SMP-$\randk_k$ & 99.24$\pm$0.02 & 73.72$\pm$0.53          & 99.71$\pm$0.01 & 82.18$\pm$0.73          & 99.71$\pm$0.61 & 84.16$\pm$0.83           \\
       & Fed-SMP-$\topk_k $    & 99.31$\pm$0.04 & 75.75$\pm$0.35          & 99.72$\pm$0.02 & 83.41$\pm$0.91          & 99.73$\pm$0.40 & 83.32$\pm$0.52           \\
EMNIST & DP-FedAvg-$\blur$      & 99.12$\pm$0.02 & 73.71$\pm$0.02          & 99.66$\pm$0.00 & 83.20$\pm$0.01          & 99.67$\pm$0.03 & 82.92$\pm$0.49           \\
       & DP-FedAvg-$\blurs$     & 99.63$\pm$0.08 & 76.25$\pm$0.35          & 99.72$\pm$0.02 & 83.41$\pm$0.91          & 99.74$\pm$0.45 & 82.92$\pm$0.49           \\
       & DP-FedSAM       & 96.28$\pm$0.64 & 76.81$\pm$0.81          & 95.07$\pm$0.45 & 84.32$\pm$0.19 & 95.61$\pm$0.94 & 85.90$\pm$0.72           \\
       & DP-FedSAM-$\topk_k $  & 94.77$\pm$0.11 & \textbf{77.27$\pm$0.67} & 95.87$\pm$1.52 & \textbf{84.80$\pm$0.60 }         & 96.12$\pm$0.85 & \textbf{87.70$\pm$0.83}  \\
\midrule
                      & DP-FedAvg                  & 93.65$\pm$0.47    & 47.98$\pm$0.24          & 93.65$\pm$0.42    & 50.05$\pm$0.47          & 93.65$\pm$0.15 & 50.90$\pm$0.86           \\
                      & Fed-SMP-$\randk_k$              & 95.46$\pm$0.43    & 48.14$\pm$0.12          & 95.36$\pm$0.06    & 51.33$\pm$0.36          & 95.36$\pm$0.06 & 50.61$\pm$0.20           \\
                      & Fed-SMP-$\topk_k $           & 95.49$\pm$0.14    & 49.93$\pm$2.29          & 95.49$\pm$0.09    & 54.11$\pm$0.83          & 95.49$\pm$0.10 & 53.30$\pm$0.45           \\
CIFAR-10              & DP-FedAvg-$\blur$              & 95.47$\pm$0.12    & 47.66$\pm$0.01          & 99.66$\pm$0.42    & 51.05$\pm$0.01          & 94.50$\pm$0.05 & 52.56$\pm$0.47           \\
                      & DP-FedAvg-$\blurs$          & 96.79$\pm$0.51    & 51.23$\pm$0.66          & 99.72$\pm$0.09    & 54.11$\pm$0.83          & 96.45$\pm$0.30 & 53.48$\pm$0.76           \\
                      & DP-FedSAM                  & 90.38$\pm$0.90    & 53.92$\pm$0.55          & 90.83$\pm$0.15    & 54.14$\pm$0.60          & 90.83$\pm$0.16 & 55.58$\pm$0.50           \\
                      & DP-FedSAM-$\topk_k $            & 93.25$\pm$0.60    & \textbf{54.85$\pm$0.86} & 92.60$\pm$0.65    & \textbf{57.00$\pm$0.69} & 91.52$\pm$0.11 & \textbf{58.82$\pm$0.51}  \\
\midrule
                      & DP-FedAvg                  & 91.14$\pm$0.16 & 16.10$\pm$0.71           & 92.33$\pm$0.08 & 15.92$\pm$0.39           & 94.01$\pm$0.10 & 17.47$\pm$0.47            \\
                      & Fed-SMP-$\randk_k$              & 90.70$\pm$0.01 & 17.25$\pm$0.16           & 92.28$\pm$0.32 & 17.50$\pm$0.19           & 94.31$\pm$0.02 & 17.68$\pm$0.44            \\
                      & Fed-SMP-$\topk_k$                & 92.58$\pm$0.24 & 18.58$\pm$0.25           & 93.51$\pm$0.11 & 18.07$\pm$0.09           & 95.06$\pm$0.05 & 19.09$\pm$0.56            \\
CIFAR-100                      & DP-FedAvg-$\blur$             & 91.27$\pm$0.01 & 17.03$\pm$0.09           & 92.33$\pm$0.03 & 17.92$\pm$0.01           & 94.01$\pm$0.04 & 18.47$\pm$0.02            \\
                      & DP-FedAvg-$\blurs$            & 92.98$\pm$0.24 & 18.98$\pm$0.25           & 94.01$\pm$0.11 & 18.27$\pm$0.19           & 95.46$\pm$0.05 & 19.59$\pm$0.06            \\
                      & DP-FedSAM                  & 82.19$\pm$0.01 & 18.88$\pm$0.31           & 85.47$\pm$0.13 & 19.09$\pm$0.15           & 87.12$\pm$0.37 & 20.64$\pm$0.48            \\
                      & DP-FedSAM-$\topk_k$              & 84.49$\pm$0.24 & \textbf{20.85$\pm$0.63}  & 88.23$\pm$0.23 & \textbf{21.24$\pm$0.69}  & 89.86$\pm$0.21 & \textbf{22.30$\pm$0.05}   \\
\bottomrule
\end{tabular}}
\vspace{-0.2cm}
\end{table*}

\begin{remark}
The proposed algorithms can achieve a tighter bound in general non-convex setting compared with previous works $\small \mathcal{O}\left( \frac{1}{\sqrt{KT}} + \frac{6K\sigma_{g}^2 + \sigma_l^2}{T} + \frac{B^2\sum_{t=1}^T(\overline{\alpha}^t + \tilde{\alpha}^t)}{T}+\frac{L^2 \sqrt{T}\sigma^2C^2pd}{m^2\sqrt{K}}\right)$ in \cite{cheng2022differentially} and $\small \mathcal{O}\left( \frac{1}{\sqrt{KT}} + \frac{3\sigma_{g}^2 + 2\sigma_l^2}{\sqrt{KT}} +\frac{4L^2 \sqrt{T}\sigma^2C^2pd}{m^2\sqrt{K}}\right)$ in \cite{hu2022federated}, and our bound reduces the impacts of the local and global variance $\sigma_l^2$, $\sigma_g^2$. 
Meanwhile, 
we are the first to theoretically analyze the impact of both the on-average norm of local updates $\overline{\alpha}^t$ and local update inconsistency among clients $\tilde{\alpha}^t$ on convergence. The negative impacts of $\overline{\alpha}^t$ and $\tilde{\alpha}^t$ are also significantly mitigated upon convergence compared with previous work \cite{cheng2022differentially} due to the local SAM optimizer adopted.
It means that we can effectively alleviate performance degradation caused by the clipping operation in DP and achieve better performance under symmetric noise. This theoretical result has also been empirically verified on several real-world data (see Sections \ref{eva} and \ref{exper_DP}).

\end{remark}

\subsection{Generalization Analysis}
To construct the upper generalization bound for our algorithms based on the previous theoretical work \cite{he2021tighter}, we prove that the FL training process in each communication round satisfies DP with the max divergence at first. Note that this DP guarantee is different from client-level DP analysis in FL (see Theorem \ref{th:privacy}) as we treat an iterative FL method as an iterative machine learning algorithm following \cite{he2021tighter}.
Below, we present the $(\tilde\varepsilon, \frac{m}{N}\delta)$-DP in the iterative communication round for Algorithm \ref{DFedSAM_DP}.

\begin{theorem}[$(\tilde\varepsilon, \frac{m}{N}\delta)$-DP in iterative communication round]\label{th:dp_round}
Suppose an FL method has $T$ communication rounds: $\{\mathbf{w}^{t}(S)\}_{t=1}^T$, where $\mathbf{w}^{t}$ is the global model. Let $\small \mathbf{v} = \frac{1}{\Vert S_{\mathcal{I}}\Vert}\sum_{z\in S_{\mathcal{I}}}g(z,\mathbf{w}^{t}) - \frac{1}{\Vert S'_{\mathcal{I}}\Vert}\sum_{z\in S'_{\mathcal{I}}}g(z,\mathbf{w}^{t})$, where $g(z,\mathbf{w}^{t})$ is a gradient value at data $z$ and $ S_{\mathcal{I}}, S'_{\mathcal{I}}$ ($\mathcal{I}$ is a mini-batch data) are two adjacent sample sets.
For any $t$-th communication round, we have the $(\tilde\varepsilon, \delta)$-differentially private guarantee, where
\begin{equation}
\small
\begin{split}
     & \tilde\varepsilon =\log \left(\frac{N-m}{N}+\frac{m}{N}\exp\left(\frac{L\rho \sqrt{2\log\frac{1}{\tilde \delta}}}{\sigma C d} + \left(\frac{\sqrt{2}L\rho}{\sigma Cd}\right)^2\right)\right),  \\
     & \delta =  \min_{1\leq t \leq T}\exp\left(-\frac{\sqrt{2}t\Vert \mathbf{v}\Vert\sigma Cd }{m} \sqrt{\log\frac{1}{\tilde\delta}}\right)\mathbb{E}(e^{t\langle \mathbf{v}, W'\rangle}).
\end{split}
\end{equation}
In the above, $\Vert \mathbf{v}\Vert \le \frac{2L\rho}{m}$ and $0<\tilde \delta\leq1$, $\sigma^2$ is the Gaussian noise variance, and $C$ and $d$ represents the clipping threshold for the local update and dimension of $\mathbf{w}^{t,k}$, respectively. 
\end{theorem}
\begin{remark}    
$\tilde\varepsilon$ and $ \delta$ are mainly decided by the number of participated clients $m$,  the clipping threshold $C$, the dimension of local update $d$, the standard deviation of DP noise $\sigma$, the L-smoothness coefficient $L$, and the perturbation radius $\rho$ in SAM optimizer. In general, the smaller the values of $\tilde\varepsilon$ and $ \delta$, the better the privacy protection ability \cite{Dwork2014the}.
\end{remark}
Let an iterative machine learning algorithm $\mathcal A$ (Algorithm \ref{DFedSAM_DP}) learn a hypothesis $\mathcal A(S)$ on a training sample set $S$ following the data distribution $\mathcal D$. 
The expected risk $\mathcal{R}_{\mathcal{D}}(\mathcal A(S))$ and empirical risk $\hat{\mathcal R}_{\mathcal{S}}(\mathcal{A}(S))$ of Algorithm \ref{DFedSAM_DP} are defined as follows:
\begin{equation*}\small
    \begin{split}
	& \mathcal{R}_{\mathcal{D}}(\mathcal{A}(S)) = \mathbb{E}_{z\sim \mathcal{D}} \ell (\mathcal{A}(S), z) = \max_{\|\delta\|_2 \leq \rho} \mathbb{E}_{z\sim \mathcal{D}} \ell (\mathbf{w} +\delta_z; z),\\
	& \hat{\mathcal R}_S(\mathcal{A}(S)) = \frac{1}{N} \sum_{i=1}^N \ell(\mathcal{A}(S), z_i) = \max_{\|\delta\|_2 \leq \rho} \frac{1}{N} \sum_{i=1}^N \ell (\mathbf{w} +\delta_{z_i}; z_i) . 
    \end{split}
\end{equation*}
Where $\ell: \mathcal H \times \mathcal Z \to \mathbb R^+$ is the loss function, $\delta=\rho\frac{g(z,\mathbf{w})}{\left \| g(z,\mathbf{w}) \right \|_2}$, and $N$ is the training sample size. Next, we give the generalization bound below.
\begin{theorem}[Generalization bound]
\label{th:gener_bound}
Under Theorem \ref{th:dp_round} and suppose the loss function $\Vert \mathcal{L}\Vert_{\infty}\le 1$, $0<\tilde \delta\leq1$ is an arbitrary positive real constant. Then, for any data distribution $\mathcal{D}$ over data space $\mathcal{Z}$, we have the following inequality:
\begin{equation}
\small
\!\!\!\mathbb{P}\!\left[\left|\hat{\mathcal{R}}_S(\mathcal{A}(S)) \!-\! \mathcal{R}_{\mathcal{D}}(\mathcal{A}(S))\right| < 4\varepsilon'\right] \!>\! 1\!-\!\frac{2e^{-1.7\varepsilon'}\delta'}{\varepsilon'} \ln \left(\frac{2}{\varepsilon'}\right),\!\!
\end{equation}
where \begin{normalsize}
    $\varepsilon'=\sqrt{2T  \log \left( \frac{1}{\tilde\delta}\right)\tilde\varepsilon^2} +T   \tilde\varepsilon \frac{e^{ \tilde\varepsilon}-1}{e^{ \tilde\varepsilon}+1}$ and 
\end{normalsize}
\begin{small}
\begin{align*}
    \delta' = & e^{-\frac{\varepsilon'+T \tilde{\varepsilon}}{2}}\left(\frac{1}{1+e^ {\tilde\varepsilon}}\left(\frac{2T {\tilde{\varepsilon}}}{T {\tilde{\varepsilon}}-\varepsilon'}\right)\right)^T\left(\frac{T {\tilde{\varepsilon}}+\varepsilon'}{T {\tilde{\varepsilon}}-\varepsilon'}\right)^{-\frac{\varepsilon'+T {\tilde{\varepsilon}}}{2 {\tilde{\varepsilon}}}}
\\
&+2-\left(1-e^{ {\tilde{\varepsilon}}}\frac{m\delta}{N(1+e^{ {\tilde{\varepsilon}}})}\right)^{\left \lceil  \frac{\varepsilon'}{ {\tilde{\varepsilon}}}\right \rceil}\left(1-\frac{m\delta}{N(1+e^{ {\tilde{\varepsilon}}})}\right)^{T-\left \lceil  \frac{\varepsilon'}{ {\tilde{\varepsilon}}}\right \rceil} 
 \\
 &- \left(1-\frac{m\delta}{N(1+e^{ {\tilde{\varepsilon}}})}\right)^{T},
\end{align*}
\end{small}
\noindent
Note that this probability inequality can be regarded as a generalization bound \cite{he2021tighter}, where
$T$ is the global training iterations (communication round) and $m$ is the global batch size (participated clients in each round) with the training sample size $N\ge\frac{2}{\varepsilon'^{2}} \ln \left(\frac{16}{e^{-\varepsilon'}\delta'}\right)$ in each local iteration. 

\end{theorem}
\begin{remark}
    The above bound indicates that the generalization performance gets better when the values of $\tilde\varepsilon$ and $ \delta$ obtained by Theorem \ref{th:dp_round} are smaller. Consequently, when the standard deviation of DP noise $\sigma$, the dimension of local update $d$, and the clipping threshold $C$ are greater,
    and the number of participated clients $m$, the L-smoothness coefficient $L$, and the perturbation radius $\rho$ in the SAM optimizer are smaller, the generalization bound gets smaller. It means that the gap between the training error and the test error is smaller, thereby achieving better generalization performance. 
\end{remark}

\section{Experiments}\label{exper}
In this section, we conduct extensive experiments to verify the effectiveness of DP-FedSAM and DP-FedSAM-$\topk_k$. 
\begin{figure*}[t]
\centering
    \subfigure[EMNIST]{
    \centering
        \includegraphics[width=1\textwidth]{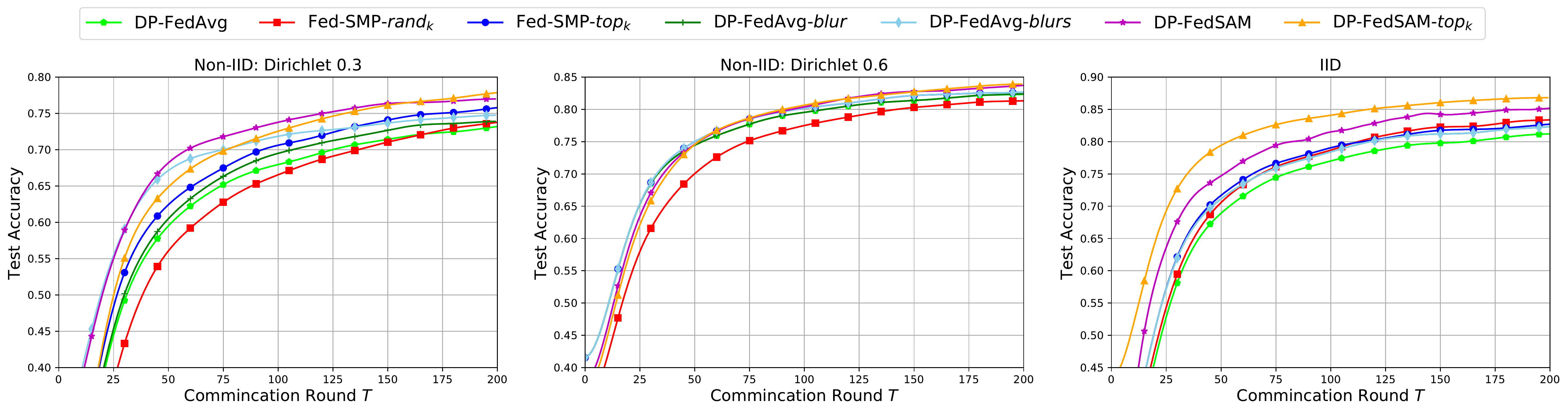}
        \label{fig:emnist}
    }
    \subfigure[CIFAR-10]{
    \centering
        \includegraphics[width=1\textwidth]{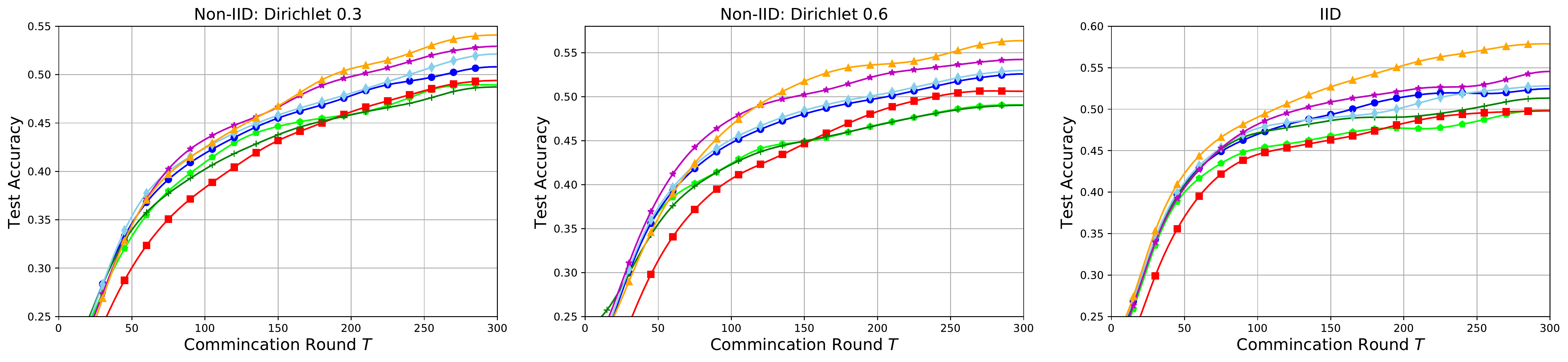}
        \label{fig:cifar10}
    }
    \centering
    \subfigure[CIFAR-100]{
    \centering
        \includegraphics[width=1\textwidth]{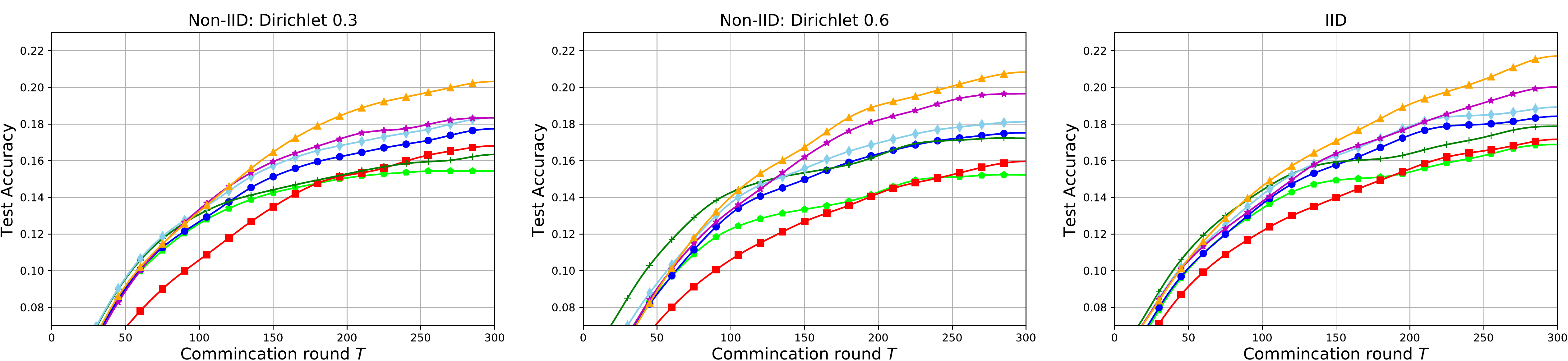}
        \label{fig:cifar100}
    }
    \caption{\small The averaged testing accuracy on \textit{EMNIST}, \textit{CIFAR-10} and \textit{CIFAR-100}  under symmetric noise for all compared methods. }
 \label{fig:all}
\end{figure*}

\subsection{Experiment Setup}

\noindent
\textbf{Dataset and Data Partition.}\  
The efficacy of DP-FedSAM is evaluated on three datasets, including \textbf{EMNIST} \cite{cohen2017emnist}, \textbf{CIFAR-10} and \textbf{CIFAR-100} \cite{krizhevsky2009learning}, in both IID and Non-IID settings. EMNIST \cite{cohen2017emnist} is a 62-class image classification dataset and we use 20\% of the dataset, which includes 88,800 training samples and 14,800 validation samples. 
Both CIFAR-10 and CIFAR-100 \cite{krizhevsky2009learning} contain 60,000 images, which are divided into 50,000 training samples and 10,000 validation samples. CIFAR-100 has finer labeling, with 100 unique labels, in comparison to CIFAR-10 with 10 unique labels. Furthermore, we distribute these datasets to each client based on Dirichlet allocation over 500 clients by default. Moreover,  Dir Partition \cite{hsu2019measuring} is used for simulating Non-IID settings across federated clients, where the local data of each client is created by sampling from the original dataset according to the label ratios based on the Dirichlet distribution Dir($\alpha$) with parameters $\alpha=0.3$ and $\alpha=0.6$.

\noindent
\textbf{Baselines.} We focus on DPFL methods that ensure client-level DP. Thus, we consider the following DPFL baselines: \textbf{DP-FedAvg} \cite{McMahan2018learning} ensures client-level DP guarantee by directly employing Gaussian mechanism to the local updates; \textbf{DP-FedAvg-blur} \cite{cheng2022differentially} adds regularization method (BLUR) based on DP-FedAvg; \textbf{DP-FedAvg-blurs} \cite{cheng2022differentially} uses local update sparsification (LUS) and BLUR for improving the performance of DP-FedAvg;
\textbf{Fed-SMP-$\randk_k$} and \textbf{Fed-SMP-$\topk_k$} \cite{hu2022federated} leverage random sparsification and $\topk_k$ sparsification technique for reducing the impact of DP noise on model accuracy, respectively.

\begin{figure*}
\centering
\subfigure[Loss landscapes: DP-FedAvg vs. DP-FedSAM on the left and DP-FedSAM vs DP-FedSAM-$\topk_k$ on the right. ]{
\includegraphics[width=0.58\textwidth]{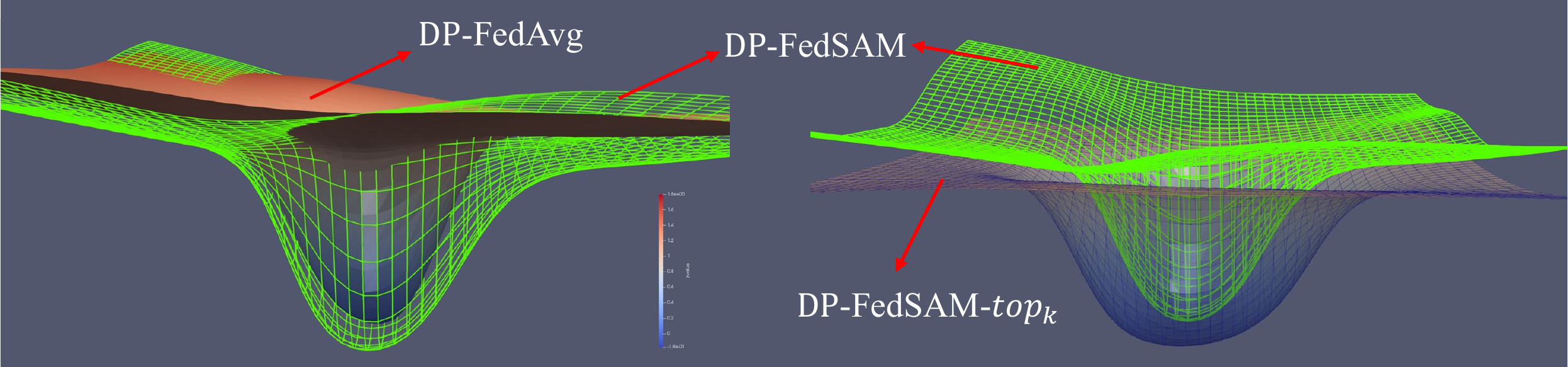}
}
\centering
\subfigure[Loss surface contour of DP-FedSAM and DP-FedSAM-$\topk_k$ is on the left and right, resptively.]{
\includegraphics[width=0.39\textwidth]{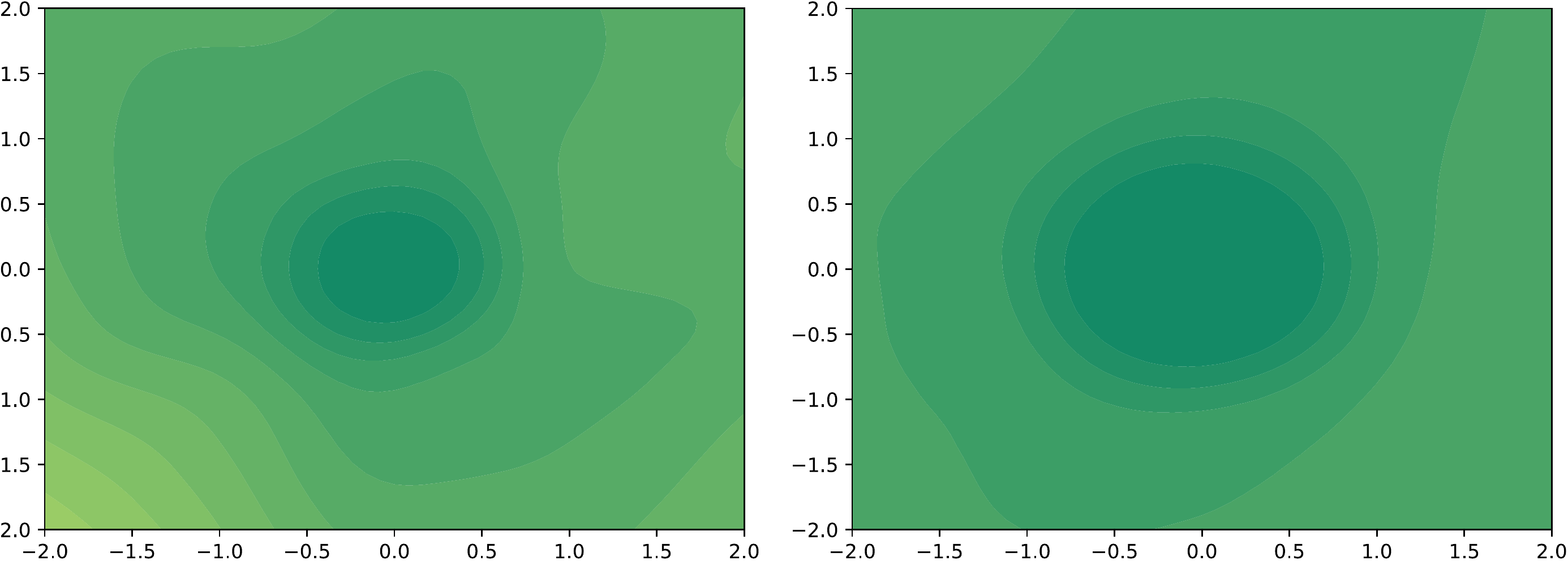}
}
\caption{Comparison of Loss landscapes (a) and surface contours (b). Compared with DP-FedAvg in the left of Figure \ref{landscape_fedavg_dpfedavg} (a) with the same setting, DP-FedSAM has a flatter landscape with both better generalization (flat minima, see the left of Figure \ref{sam_land} (a)) and higher weight perturbation robustness (see the left of Figure \ref{sam_land} (b)). Meanwhile, DP-FedSAM-$\topk_k$ also features similar advantages compared with DP-FedSAM in the right of Figure \ref{sam_land} (a) and (b).
}
\label{sam_land}
\end{figure*}

\noindent
\textbf{Configuration.}
For EMNIST, we use a simple CNN model and train for $200$ communication rounds. For CIFAR-10 and CIFAR-100 datasets, we use the ResNet-18 \cite{he2016deep} backbone and train for $300$ communication rounds. 
In all experiments, we set the number of clients $M$ to $500$. 
For the EMNIST dataset, we set the mini-batch size to 32 and train with a simple CNN model, which includes two convolutional layers with 5×5 kernels, max pooling, followed by a 512-unit dense layer. For CIFAR-10 and CIFAR-100 datasets, we set the mini-batch size to 50 and train with ResNet-18 \cite{he2016deep} architecture. 
For each algorithm and each dataset, the learning rate is set via grid search within $\{10^{-0.5}, 10^{-1}, 10^{-1.5}, 10^{-2}\}$. The weight perturbation ratio $\rho$ is set via grid search within $\{0.01, 0.1, 0.3, 0.5, 0.7, 1.0\}$. For all methods using the sparsification technique, the sparsity ratio is set to $p=0.4$.
The default sample ratio $q$ of the client is $0.1$. The local learning rate $\eta$ is set to 0.1 with a decay rate $0.0005$ and momentum $0.5$, and the number of training epochs is $30$. For privacy parameters, the noise multiplier $\sigma$ is set to $0.95$ and  the privacy failure probability $\delta=\frac{1}{M}$. The clipping threshold $C$ is selected by grid search within $\{0.1,0.2,0.4,0.6,0.8\}$, and we find that the gradient explosion phenomenon can occur when $C \ge 0.6$ on EMNIST and $C=0.2$ performs best on three datasets.  The weight perturbation ratio is set to $\rho = 0.5$. We run each experiment for $3$ trials and report the best averaged testing accuracy in each experiment. 


\subsection{Experiment Evaluation}\label{eva}


\begin{table*}
\centering
\caption{Performance comparison under different privacy budgets $\epsilon$ on \emph{CIFAR-10} and \emph{CIFAR-100}.}
\vspace{-0.2cm}
\label{privacy}
\scriptsize
\renewcommand{\arraystretch}{0.5}
\resizebox{0.8\linewidth}{!}{
\begin{tabular}{cccccc} 
\toprule
\multirow{2}{*}{Task} &\multirow{2}{*}{Algorithm} & \multicolumn{4}{c}{Averaged test accuracy (\%) under different privacy budgets $\epsilon$}                                                                                                            \\ 
\cmidrule{3-6}
&  &  $\epsilon$ = 4  & $\epsilon$ = 6 &$\epsilon$ =  8  & $\epsilon$ = 10     \\                              
\midrule
\multirow{7}{*}{CIFAR-10}  
&DP-FedAvg                 & 38.23 $\pm$ 0.15    & 43.87 $\pm$ 0.62  &  46.74 $\pm$ 0.03  & 49.06 $\pm$ 0.49  \\
&Fed-SMP-$\randk_k$        & 33.78 $\pm$ 0.92    & 42.21 $\pm$ 0.21  &  48.20 $\pm$ 0.05  & 50.62 $\pm$ 0.14  \\
&Fed-SMP-$\topk_k$        & 38.99 $\pm$0.50     & 46.24 $\pm$ 0.80  & 49.78 $\pm$ 0.78   & 52.51 $\pm$ 0.83  \\
&DP-FedAvg-$\blur $       & 38.23 $\pm$ 0.70    & 43.93 $\pm$ 0.48  & 46.74 $\pm$ 0.92   & 49.06 $\pm$ 0.13  \\
&DP-FedAvg-$\blurs $      & 39.39 $\pm$0.43     & 46.64 $\pm$ 0.36  & 50.18 $\pm$ 0.27   & 52.91 $\pm$ 0.57  \\
&DP-FedSAM                 & \textbf{39.89$\pm$ 0.17}   & 47.92 $\pm$ 0.23           & 51.30 $\pm$ 0.95          & 53.18 $\pm$ 0.40  \\
&DP-FedSAM-$\topk_k$        & 38.96 $\pm$ 0.61           & \textbf{49.17 $\pm$ 0.15}   & \textbf{53.64 $\pm$ 0.12} & \textbf{56.36 $\pm$ 0.36}  \\
\midrule
\multirow{7}{*}{CIFAR-100} 
&DP-FedAvg               & 9.65 $\pm$ 0.34  & 12.81 $\pm$ 0.29 &  14.30 $\pm$ 0.05  & 15.23 $\pm$ 0.24  \\
&Fed-SMP-$\randk_k$        & 7.95 $\pm$ 0.91   & 11.51 $\pm$ 0.18  &  14.20 $\pm$ 0.90  & 15.92 $\pm$ 1.03 \\
&Fed-SMP-$\topk_k$       & 9.90 $\pm$ 0.89    & 14.22 $\pm$ 0.82       & 16.57 $\pm$ 0.75   & 17.71 $\pm$ 0.36  \\
&DP-FedAvg-$\blur $        & 9.90 $\pm$ 0.89   & 14.22 $\pm$ 0.82  & 16.57 $\pm$ 0.75   & 17.71 $\pm$ 0.36   \\
&DP-FedAvg-$\blurs $     & 9.65 $\pm$ 0.34     & 12.81 $\pm$ 0.29  & 14.30 $\pm$ 0.05   & 15.23 $\pm$ 0.24 \\
&DP-FedSAM                  & 10.03 $\pm$ 0.63   & 14.46 $\pm$ 1.21           & 18.20 $\pm$ 1.34     & 19.65 $\pm$ 0.80  \\
&DP-FedSAM-$\topk_k$      & \textbf{10.08 $\pm$ 0.68 }       & \textbf{15.26 $\pm$ 0.78}   & \textbf{18.99 $\pm$ 1.38} & \textbf{20.79 $\pm$ 1.07}  \\
\bottomrule
\end{tabular}}
\end{table*}

\noindent
\textbf{Overall performance comparison.}
In Table \ref{table:all_baselines} and Figure \ref{fig:all}, we evaluate DP-FedSAM and DP-FedSAM-$\topk_k$ on EMNIST, CIFAR-10, and CIFAR-100 in both settings compared with all baselines. It is clear that our proposed algorithms consistently outperform other baselines under symmetric noise in terms of accuracy and generalization.
This fact indicates that we significantly improve the performance and generate better trade-off between performance and privacy in DPFL.
For instance, the averaged testing accuracies are $85.90\%$ in DP-FedSAM and $87.70\%$ in DP-FedSAM-$\topk_k$ on EMNIST in the IID setting, which are better than other baselines. Meanwhile, the differences between training accuracy and test accuracy are $9.71\%$ in DP-FedSAM and $8.40\%$ in DP-FedSAM-$\topk_k$, in comparison to $17.74\%$ in DP-FedAvg and $16.41\%$ in Fed-SMP-$\topk_k$, respectively. Consequently, it shows that our algorithms significantly mitigate the performance degradation issue caused by DP.

\noindent
\textbf{Impact of Non-IID levels.}
In the experiments under different participation cases as shown in Table \ref{table:all_baselines}, we further demonstrate the robustness of the proposed algorithms in generalization. The heterogeneous data distribution of local clients is set to various participation levels including IID, Dirichlet 0.6, and Dirichlet 0.3, making the training of the global model more challenging. On EMNIST, as the Non-IID level decreases, DP-FedSAM achieves better generalization than DP-FedAvg, and the differences between training and test accuracies in DP-FedSAM $(19.47\%, 10.75\%, 9.71\%)$ are lower than those in DP-FedAvg $(26.18\%, 17.35\%, 17.74\%)$. Similarly, the differences in DP-FedSAM-$\topk_k$ $(17.50\%, 11.07\%, 8.40\%)$ are also lower than those in Fed-SMP-$\topk_k$ $(23.56\%, 16.31\%, 16.41\%)$. These observations confirm that our algorithms are more robust than baselines in various degrees of heterogeneous data.

\begin{figure}
\centering
\includegraphics[width=0.48\textwidth]{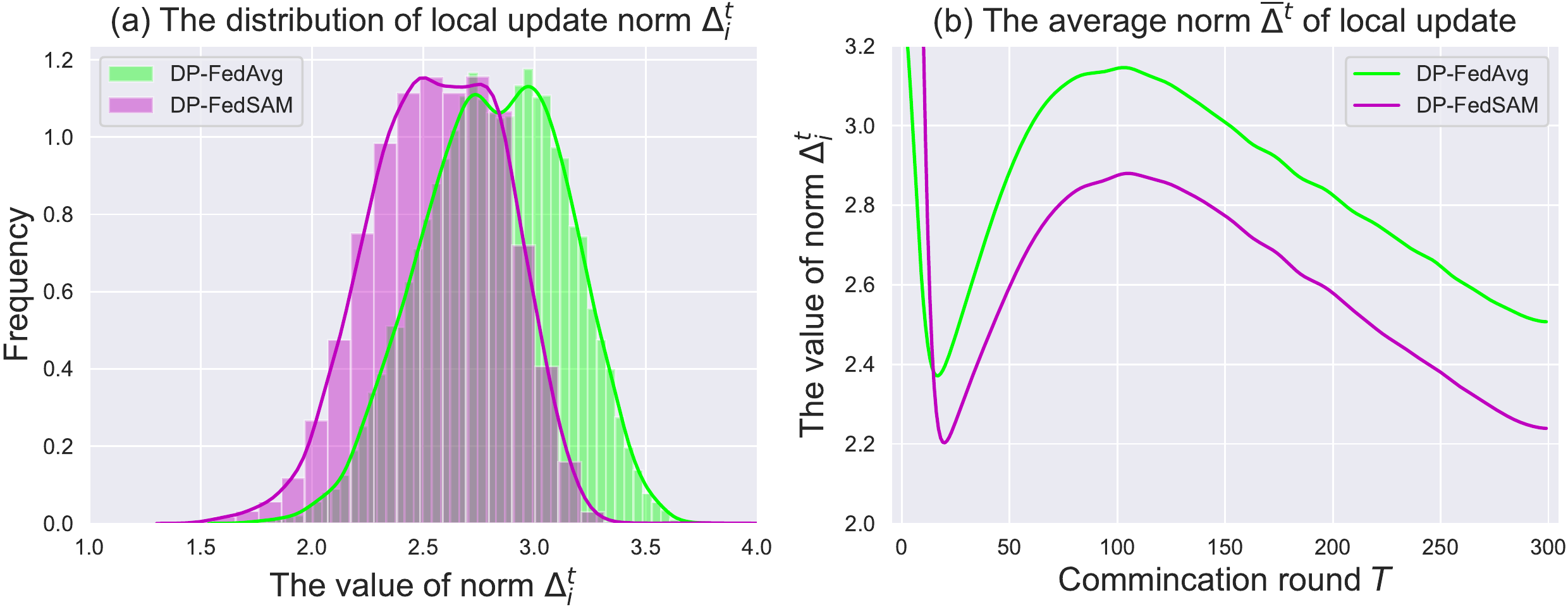}

\vspace{-0.2cm}
\caption{\small Norm distribution and average norm of local updates.
}
\vspace{-0.4cm}
\label{fig:norm}
\end{figure}

\subsection{Discussion on DP with SAM in FL}\label{exper_DP}
In this subsection, we discuss how SAM mitigates the negative impacts of DP from the aspects of the norm of local update and the visualization of the loss landscape and contour. Meanwhile, we also investigate the training performance with SAM under different privacy budgets $\epsilon$ compared with all baselines. These experiments are conducted on CIFAR-10 with ResNet-18 \cite{he2016deep} and Dirichlet $\alpha=0.6$.

\begin{table*}
\centering
\caption{Performance comparison under different sparsity ratio $p$.}
\label{sparsity_ratio}
\scriptsize
\renewcommand{\arraystretch}{1}
\label{tb:various_spar}
\resizebox{0.95\linewidth}{!}{
\begin{tabular}{cccccccc} 
\toprule
\multirow{2}{*}{Task}                                                 & \multirow{2}{*}{Performance} & \multicolumn{6}{c}{Different sparsity ratio $p$}  \\ 
\cmidrule{3-8}
                                                                      &                             & $p=1.0$ & $p=0.1$ & $p=0.2$ & $p=0.4$ & $p=0.6$ & $p=0.8$              \\ 
\midrule
\multirow{3}{*}{\begin{tabular}[c]{@{}c@{}}CIFAR-10 \\\end{tabular}}                                            & Train (\%)                &    90.83 $\pm$ 0.15   & 91.35 $\pm$ 0.66   &  92.83 $\pm$ 0.09  &  92.60 $\pm$ 0.65   &  92.43 $\pm$ 0.37  &     91.89 $\pm$  0.17         \\
    & Validation (\%)       &   54.14 $\pm$ 0.60      &  55.05 $\pm$ 0.74  &  56.12 $\pm$  0.41  & \textbf{57.00 $\pm$ 0.69}    &  56.42 $\pm$  0.13   &      56.14 $\pm$  0.32       \\
\cmidrule{2-8}
    & \begin{tabular}[c]{@{}c@{}}Gain (\%) compared\\ with $p=1.0$\end{tabular}         &   0.00  &   0.91 $\pm$ 0.14 &  1.98 $\pm$ 0.19&  \textbf{2.86 $\pm$ 0.09}   &  2.28 $\pm$ 0.47  &     2.00 $\pm$ 0.28         \\ 
\midrule
\multirow{3}{*}{\begin{tabular}[c]{@{}c@{}}CIFAR-100 \\\end{tabular}} & Train (\%)        &     85.47 $\pm$ 0.13        &  86.29 $\pm$ 0.26  &  87.49 $\pm$ 0.30  &  88.23 $\pm$ 0.23   &  86.41 $\pm$ 0.37   &    85.66 $\pm$  0.14       \\
    & Validation (\%)     &      19.09 $\pm$ 0.15         &   20.38 $\pm$ 0.24 &   20.62 $\pm$ 0.85  &   \textbf{21.24 $\pm$ 0.69}  &    20.55 $\pm$ 0.84   &      19.79 $\pm$ 0.28      \\
\cmidrule{2-8}
    &  \begin{tabular}[c]{@{}c@{}}Gain (\%) compared\\ with $p=1.0$\end{tabular}        &  0.00   &  1.29 $\pm$ 0.09  &  1.53 $\pm$ 0.80   &   \textbf{2.15 $\pm$ 0.54} &  1.46 $\pm$ 0.69  &    0.70 $\pm$ 0.13            \\
\bottomrule
\end{tabular}}
\end{table*}

\noindent
\textbf{Loss landscape and contour.}
To visualize the structure of the minima and investigate the robustness to DP noise by DP-FedSAM and DP-FedSAM-$\topk_k$ ($p=0.4$) compared with DP-FedAvg, we show the loss landscapes and surface contours \cite{li2018visualizing} in Figure \ref{sam_land}. It is clear that DP-FedSAM features flatter minima and better robustness to DP noise than DP-FedAvg in the left Figure \ref{landscape_fedavg_dpfedavg} (a). Moreover, DP-FedSAM-$\topk_k$ also features flatter minima and better robustness to DP noise than DP-FedSAM in the right of Figure \ref{landscape_fedavg_dpfedavg} (a) and (b). This shows that the flatter landscape and better generalization ability can be achieved by using the SAM local optimizer and local update sparsification techniques in DPFL.
Furthermore, it also indicates that our proposed algorithms achieve better generalization and makes the training process more suitable for the DPFL setting.

\noindent
\textbf{The norm of local update.}
To validate the theoretical results on mitigating the adverse impacts of the norm of local updates, we conduct experiments on DP-FedSAM and DP-FedAvg with clipping threshold $C=0.2$ as shown in Figure \ref{fig:norm}. We show the norm $\Delta_i^t$ distribution and average norm $\overline{\Delta}^t$ of local updates before clipping during the communication rounds. In contrast to DP-FedAvg, most of the norm is distributed over smaller values in our scheme according to Figure \ref{fig:norm} (a), which means that the clipping operation drops less information. Meanwhile, the on-average norm $\overline{\Delta}^t$ is smaller than DP-FedAvg as shown in Figure \ref{fig:norm} (b). These observations are also consistent with our theoretical results in Section \ref{th}.

\subsection{Performance under Different Privacy Budgets}
Table \ref{privacy} shows the test accuracies under various privacy budgets $\epsilon$ on both CIFAR-10 and CIFAR-100 datasets.
Specifically, on CIFAR-10, DP-FedSAM and DP-FedSAM-$\topk_k$ significantly outperform DP-FedAvg and Fed-SMP-$\topk_k$ by $1\% \sim 4\%$ and $3\% \sim 4\%$ under the same $\epsilon$, respectively. On the more complex CIFAR-100 dataset, our algorithms also have significant performance improvement. That is DP-FedSAM and DP-FedSAM-$\topk_k$ significantly improve the accuracy of DP-FedAvg and Fed-SMP-$\topk_k$ by $1\% \sim 4\%$ and $1\% \sim 3\%$ under the same $\epsilon$, respectively.
Furthermore, the test accuracy tends to improve as the privacy budget $\epsilon$ increases, which suggests that a proper balance is to be maintained between training performance and privacy. 

\subsection{Discussion on Local Update Sparsification}
In Table \ref{tb:various_spar}, we investigate the impact of various sparsity ratio $p$ on the model performance improvement on both CIFAR-10 and CIFAR-100 datasets. Specifically, we can see that the validation accuracy (test accuracy) is improved as the value of $p$ increases from $0.1$ to $0.4$, and then accuracy is degraded as the value of $p$ increases from $0.4$ to $1.0$. It indicates where having an optimal value of sparsity ratio $p$. The reason lies in the balance between the information error introduced by the sparsification operation and the magnitude of random noise caused by the DP operation. When the value of $p$ is small enough, the information error is large so that this error is a more significant factor in the ill impact on performance than the random noise in DP due to lots of parameters with added noise being discarded. Meanwhile, when the value of $p$ is large enough, the random noise is a more important factor in the ill impact on performance than the information error. Therefore, a proper trade-off between  the information error and the magnitude of random noise is achieved when the value of $p$ is $0.4$. Note that DP-FedSAM is the same as DP-FedSAM-$\topk_k$ when $p=1.0$.

\subsection{Ablation Study}
In this subsection, we verify the effect of each component and hyper-parameter in DP-FedSAM. All the ablation studies are conducted on EMNIST with Dirichlet 0.6. 
\begin{figure*}[ht]
\centering
\includegraphics[width=1\textwidth]{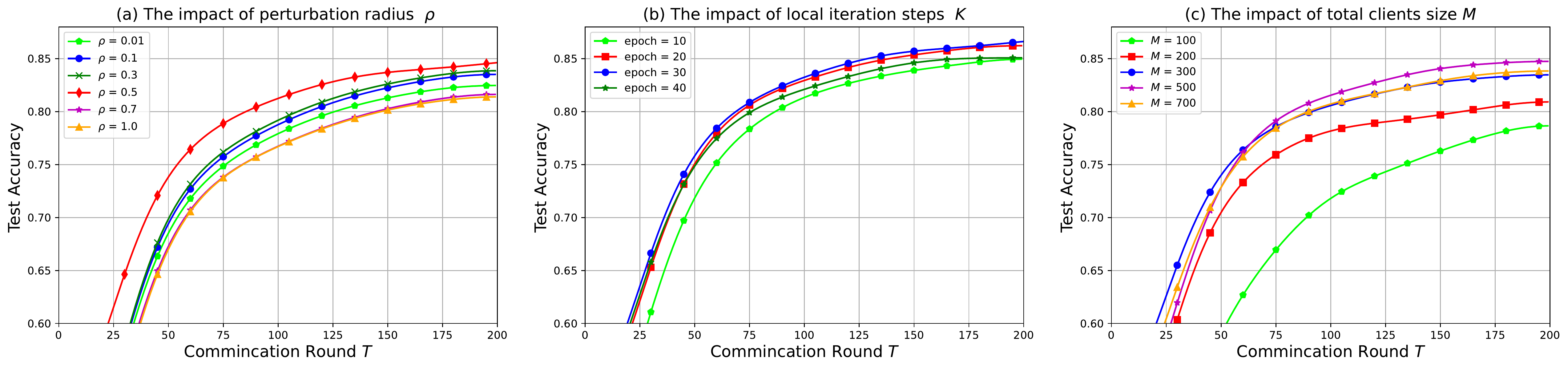}
\centering
\caption{Impact of hyper-parameters: perturbation radius $\rho$, local iteration steps $K$, total clients size $M$.}
\label{fig:abla}
\end{figure*}

\noindent
\textbf{Perturbation weight $\rho$.}
Perturbation weight $\rho$ has an impact on performance as the added perturbation is accumulated when the communication round $T$ increases. 
To select a proper value for our algorithms, we conduct some experiments on various perturbation radius within $\{ 0.01, 0.1, 0.3, 0.5, 0.7, 1.0\}$ in Figure \ref{fig:abla} (a), with $\rho = 0.5$, we achieve better convergence and performance. 

\noindent
\textbf{Local iteration steps $K$.}
Large local iteration steps $K$ can help the convergence in previous DPFL work \cite{cheng2022differentially} with the theoretical guarantees. To investigate the acceleration on $T$ by adopting a larger $K$, we fix the total batchsize and change local training epochs. In Figure \ref{fig:abla} (b), our algorithm can accelerate the convergence in Theorem \ref{th:conver} as a larger $K$ is adopted, that is, use a larger epoch value. However, the adverse impact of clipping on training increases as $K$ is too large, for instance, epoch 
= $40$.

\noindent
\textbf{Client size $M$.}
We compare the performance with different numbers of client participation $m=\{ 100, 200, 300, 500, 700\}$ in Figure \ref{fig:abla} (c). 
In general, smaller $m$ values tend to produce worse performance due to the large variance $\sigma^2 C^2/m$ in DP noise with the same setting. Meanwhile, when $m$ is too large such as $M=700$, the performance may degrade as the local data size decreases. 
\begin{table}
\caption{The averaged training accuracy and testing accuracy.}
\vspace{-0.2cm}
\label{ab:sam}
\centering
\renewcommand{\arraystretch}{1}
\resizebox{1\linewidth}{!}{
\begin{tabular}{cccc}
\toprule
Algorithm  & Train  (\%)& Validation (\%) & Differential value (\%)  \\
\midrule
DP-FedAvg  &    99.55$\pm$0.02   &       82.20$\pm$ 0.35        &   17.35$\pm$0.32              \\
DP-FedSAM  &    95.07$\pm$0.45   &    84.32$\pm$0.19 $\uparrow$        &     10.75$\pm$0.26              $\downarrow$  \\
Fed-SMP-$\topk_k$   &   99.72$\pm$0.02    &     83.41  $\pm$ 0.91     &        16.31$\pm$ 0.89            \\
DP-FedSAM-$\topk_k$ & 95.87$\pm$0.52        &     84.80$\pm$0.60  $\uparrow$      & 11.07$\pm$0.08  $\downarrow$  \\   
\bottomrule
\end{tabular}}
\vspace{-0.4cm}
\end{table}

\noindent
\textbf{Effect of SAM.}
As shown in Table \ref{ab:sam}, it is clear that DP-FedSAM and DP-FedSAM-$\topk_k$ can achieve noticeable performance improvement and better generalization compared with DP-FedAvg and Fed-SMP-$\topk_k$ when the SAM optimizer is adopted.

\section{Conclusion}
In this paper, we focus on the challenging issue of severe performance degradation caused by dropped model information and exacerbated model inconsistency. The key contribution is that we are the first to alleviate this issue from the optimizer perspective and propose two novel and effective frameworks DP-FedSAM and DP-FedSAM-$\topk_k$ with a flatter loss landscape and better generalization. Meanwhile, we present the detailed analysis of how SAM mitigates the adverse impacts of DP and achieve a 
tighter bound on convergence, and also deliver the sensitivity, privacy, and generalization analysis. 
Moreover, we investigate the combination of the SAM local optimizer and sparsification strategy, 
which brings the benefits of a flatter landscape and better generalization. Furthermore, 
we present the first analysis on the combined impacts of the on-average norm of local updates and local update consistency among clients on training and provide corresponding experimental evaluations.
Finally, empirical results also verify the superiority of our approaches on several real-world datasets against advanced baselines.

{
\bibliographystyle{ieeetr}
\bibliography{ref}
}

\clearpage
\newpage
\onecolumn 
\appendices

\vspace{0.5in}
\begin{center}
 \rule{6.875in}{0.7pt}\\ 
 {\Large\bf Supplementary Material for\\ ``Towards the Flatter Landscape and Better Generalization in Federated \\Learning under Client-level Differential Privacy''}
 \rule{6.875in}{0.7pt}
\end{center}

\section{Main Proof} \label{appendix_th}
\subsection{Notations and Preliminaries} \label{ap:np}

$S = \{(x_1, y_1), \ldots, (x_N, y_N) | x_i \in \mathcal X \subset \mathbb R^{d_X},$ $ y_i \in \mathcal Y \subset \mathbb R^{d_Y}, i = 1, \ldots, N\}$ is a training sample set, where $x_i$ is the $i$-th feature, $y_i$ is the corresponding label, and $d_X$ and $d_Y$ are the dimensions of the feature and the label, respectively. For the brevity, we define $z_i = (x_i, y_i)$. We also define random variables $Z = (X, Y)$, such that all $z_i = (x_i, y_i)$ are independent and identically distributed (i.i.d.) observations of the variable $Z = (X, Y) \in \mathcal Z,~ Z \sim \mathcal D$, where $\mathcal D$ is the data distribution.
For a machine learning algorithm $\mathcal A$, it learns a hypothesis $\mathcal A(S)$, $\mathcal A(S) \in \mathcal H \subset \mathcal Y^{\mathcal X} = \{f: \mathcal X \to \mathcal Y\}$. 

The expected risk $\mathcal{R}_{D}(\mathcal A(S))$ and empirical risk $\hat{\mathcal R}_{\mathcal{S}}(\mathcal{A}(S))$ of the algorithm $\mathcal{A}$ are defined as follows,
\begin{gather*}
	 \mathcal{R}_{\mathcal{D}}(\mathcal{A}(S)) = \mathbb{E}_{z\sim \mathcal{D}} \ell (\mathcal{A}(S), z) = \max_{\|\delta\|_2 \leq \rho} \mathbb{E}_{z\sim \mathcal{D}} \ell (\mathbf{w} +\delta_z; z),\\
	\hat{\mathcal R}_S(\mathcal{A}(S)) = \frac{1}{N} \sum_{i=1}^N \ell(\mathcal{A}(S), z_i) = \max_{\|\delta\|_2 \leq \rho} \frac{1}{N} \sum_{i=1}^N \ell (\mathbf{w} +\delta_{z_i}; z_i),
\end{gather*}
where $\ell: \mathcal H \times \mathcal Z \to \mathbb R^+$ is the loss function, $\delta=\rho\frac{g(z,\mathbf{w})}{\left \| g(z,\mathbf{w}) \right \|_2}$ ($g(z,\mathbf{w})$ is the gradient), and $N$ is the training sample size. 

\begin{definition}[Generalization bound, \cite{he2021tighter}]
\label{ap:gb}
The generalization error is defined as the difference between the expected risk and empirical risk, 
\begin{equation*}
\operatorname{Gen}_{S,\mathcal{A}(S)}\overset{\triangle}{=} \mathcal{R}_{\mathcal{D}}(\mathcal{A}(S)) - \hat{\mathcal R}_S (\mathcal{A}(S)),
\end{equation*}
whose upper bound is called the generalization bound.

\end{definition}

\begin{definition}[Differential Privacy, \cite{Dwork2014the}]
\label{ap:dp}
hhhA stochastic algorithm $\mathcal{A}$ is called ($\varepsilon,\delta$)-differentially private if for any hypothesis subset $\mathcal{H}_0 \subset \mathcal H$ and any neighboring sample set pair $S$ and $S'$ which differ by only one example (called $S$ and $S'$ adjacent), we have
\begin{equation*}
\label{eq:dp}
\log \left[ \frac{\mathbb P_{\mathcal{A}(S)}(\mathcal{A}(S)\in \mathcal{H}_0) - \delta}{\mathbb P_{\mathcal{A}(S')}(\mathcal{A}(S')\in \mathcal{H}_0)} \right] \le \varepsilon.
\end{equation*}
The algorithm $\mathcal{A}$ is also called $\varepsilon$-differentially private, if it is $(\varepsilon, 0)$-differentially private.
\end{definition}

\begin{definition}[Multi-Sample-Set Learning Algorithms, \cite{he2021tighter}]
Suppose the training sample set $S$ with size $kN$ is separated to $k$ sub-sample-sets $S_1, \ldots, S_k$, each of which has the size of $N$. In another word, $S$ is formed by $k$ sub-sample-sets as
\begin{equation*}
S = (S_1, \ldots, S_k).
\end{equation*}
The hypothesis $\mathcal{B}(S)$ learned by {\it multi-sample-set algorithm} $\mathcal{B}$ on dataset $S$ is defined as follows,
\begin{equation*}
\mathcal{B}: \mathcal Z^{k\times N}\mapsto \mathcal H \times \{1, \ldots, k\},~ \mathcal{B}(S)=  \left(h_{\mathcal B(S)}, i_{\mathcal B(S)} \right).
\end{equation*}
\end{definition}

\subsection{Preliminary Lemmas}
\begin{lemma}[Lemma B.1, \cite{Qu2022Generalized}] 
\label{e_delta}
Under Assumptions~\ref{a1}-\ref{a2}, the updates for any learning rate satisfying $\eta \leq \frac{1}{4KL}$ have the drift due to $\delta_{i,k} - \delta$:
\begin{equation}
    \frac{1}{M}\sum_{i}\mathbb{E} [\|\delta_{i,k} - \delta \|^2 ] \leq 2K^2 L^2 \eta^2 \rho^2 . \nonumber
\end{equation}
Where 
\begin{equation}
		\delta = \rho \frac{\nabla f(\mathbf{w}^t)}{\|\nabla f(\mathbf{w}^t)\|}, ~~~ \delta_{i,k} = \rho \frac{\nabla F_i (\mathbf{w}^{t,k} ,\xi_i )}{\|\nabla F_i (\mathbf{w}^{t,k}, \xi_i )\|}. \nonumber
\end{equation}
\end{lemma}
\begin{lemma}[Lemma B.2, \cite{Qu2022Generalized}]
\label{e_w}
Under the above assumptions, the updates for any learning rate satisfying $\eta_l \leq \frac{1}{10KL}$ have the drift due to $\mathbf{w}^{t,k}(i)  - \mathbf{w}^t$:
\begin{equation}
\begin{split}
    \frac{1}{M}\sum_{i}\mathbb{E} [\|\mathbf{w}^{t,k}(i) - \mathbf{w}^t \|^2 ] 
    & \leq 5K\eta^2 \Big(2L^2 \rho^2 \sigma_l^2
    + 6K(3\sigma_g^2 + 6L^2 \rho^2 )  + 6K\|\nabla f(\mathbf{w}^t)\|^2 \Big) + 24K^3 \eta^4 L^4 \rho^2 . \nonumber
\end{split}
\end{equation}
\end{lemma}

\begin{lemma}[Theorem 9, \cite{balle2018privacy}] 
\label{lemma:one}
This lemma provide bound of differential privacy parameters after sub-sampling uniformly without replacement. Let $\mathcal{M}^o:\mathcal{Z}^m\mapsto \Delta \mathcal{H}$ be any mechanism preserving $(\varepsilon,\delta)$ differential privacy. Let   $\mathcal{M}^{wo}:\mathcal{Z}^N\mapsto \Delta \mathcal{Z}^m$ be the uniform sub-sampling without replacement mechanism. Then mechanism $\mathcal{M}^o\circ\mathcal{M}^{wo}$ satisfy $(\log (1+(m / N)(e^{\varepsilon}-1)),m\delta/N)$  differential privacy.
\end{lemma}

\begin{lemma}[Theorem 4, \cite{he2021tighter}] 
\label{lemma:multi}
This lemma gives the relationship between one step privacy preserving methods and iterative machine learning methods. Suppose an iterative machine learning algorithm $\mathcal A$ has $T$ steps: $\left\{W_i(S)\right\}_{i=1}^T$. Specifically, we define the $i$-th iterator as follows,
	\begin{equation*}
	\mathcal{M}_i: (W_{i-1}(S), S) \mapsto W_{i}(S).
	\end{equation*}
	Assume that $W_0$ is the initial hypothesis (which does not depend on $S$). If for any fixed $W_{i-1}$, $\mathcal{M}_i(W_{i-1},S)$ is ($\varepsilon_i, \delta$)-differentially private, 
	then $\left\{W_i\right\}_{i=0}^T$ is ($\varepsilon'$, $\delta'$)-differentially private that
	\begin{equation*}
	\varepsilon^{\prime}=\sqrt{2  \log \left( \frac{1}{\tilde\delta}\right)\left(\sum\limits_{i=1}^T \varepsilon_{i}^2\right)} +\sum\limits_{i=1}^T \varepsilon_i \frac{e^{\varepsilon_i}-1}{e^{\varepsilon_i}+1},
	\end{equation*}
	\begin{align*}
    \delta' = & e^{-\frac{\varepsilon'+T\varepsilon}{2}}\left(\frac{1}{1+e^\varepsilon}\left(\frac{2T\varepsilon}{T\varepsilon-\varepsilon'}\right)\right)^T\left(\frac{T\varepsilon+\varepsilon'}{T\varepsilon-\varepsilon'}\right)^{-\frac{\varepsilon'+T\varepsilon}{2\varepsilon}}- \left(1-\frac{\delta}{1+e^{\varepsilon}}\right)^{T}
\\
&+2-\left(1-e^{\varepsilon}\frac{\delta}{1+e^{\varepsilon}}\right)^{\left \lceil  \frac{\varepsilon'}{\varepsilon}\right \rceil}\left(1-\frac{\delta}{1+e^{\varepsilon}}\right)^{T-\left \lceil  \frac{\varepsilon'}{\varepsilon}\right \rceil} .
\end{align*}
Note that $0 < \tilde{\delta} \leq 1$ is an arbitrary positive real constant, which is also the same as $\delta$ in definition \ref{def:max_diver}.
\end{lemma}

\begin{lemma}[Theorem 1, \cite{he2021tighter}] 
\label{lemma:g-dp}
This lemma gives a high-probability generalization bound for any ($\varepsilon,\delta$)-differentially private machine learning algorithm. Suppose algorithm $\mathcal{A}$ is ($\varepsilon,\delta$)-differentially private, the training sample size $N\ge
\max\left\{\frac{c_1}{\varepsilon^2}\ln{\left (\frac{c_2}{e^{-c_3\varepsilon }\delta } \right )} , \frac{c_4}{c_5(1-c_6e^{-\varepsilon})}\ln c_7e^{-\varepsilon}\delta\right\}$, and the loss function $\Vert l\Vert_{\infty}\le 1$. Then, for any data distribution $\mathcal{D}$ over data space $\mathcal{Z}$, we have the following inequality,
\begin{equation*}
\mathbb{P}\left[\left|\hat{\mathcal{R}}_S(\mathcal{A}(S)) - \mathcal{R}(\mathcal{A}(S))\right| < 4\varepsilon\right] > 1-\frac{2e^{-1.7\varepsilon}\delta}{\varepsilon} \ln \left(\frac{2}{\varepsilon}\right).
\end{equation*}
Where all $c_1, ..., c_7$ are some positive constants.
\end{lemma}

\begin{lemma}\label{y_k-x_k}
The two model parameters conducted by two adjacent datasets which differ only one sample from client $i$ in the communication round $t$, 
\begin{equation}
    \sum_{k=0}^{K-1}\|\mathbf{y}^{t,k}(i)- \mathbf{x}^{t,k}(i) \|_2^2
    \leq 2K \max \|\Delta_i^t(\mathbf{y})-\Delta_i^t(\mathbf{x})\|_2^2 \nonumber.
\end{equation}
\end{lemma}
\begin{proof}
Recall the local update from client $i$ is $\sum_{k=0}^{K-1}\mathbf{w}^{t,k}(i) = \sum_{k=0}^{K-1}\mathbf{w}^{t,k-1}(i) + \Delta_i^t $, (the initial value is assumed as $ \mathbf{w}^{t, -1}= \mathbf{w}^{t, 0} =\mathbf{w}^{t}$). Then,
\begin{equation}
    \begin{split}
      & \sum_{k=0}^{K-1}\|\mathbf{y}^{t,k}(i)- \mathbf{x}^{t,k}(i) \|_2^2
    \leq 2 \sum_{k=0}^{K-1}\|\mathbf{y}^{t,k-1}(i)- \mathbf{x}^{t,k-1}(i) \|_2^2\\
    &+ 2 \|\Delta_i^t(\mathbf{y})-\Delta_i^t(\mathbf{x})\|_2^2. \nonumber
    \end{split}
\end{equation}
The recursion from $\tau=0$ to $k$ yields
\begin{equation}
    \sum_{k=0}^{K-1}\|\mathbf{y}^{t,k}(i)- \mathbf{x}^{t,k}(i) \|_2^2
    \overset{a)}{\leq} 2K \max \|\Delta_i^t(\mathbf{y})-\Delta_i^t(\mathbf{x})\|_2^2 \nonumber.
\end{equation}
Where a) uses the initial value $\mathbf{w}^t(i)=\mathbf{x}^{t, 0}(i) =\mathbf{y}^{t, 0}(i)$ and $0<k \leq K$.
\end{proof}

\begin{lemma}\label{Delta_average}
Under assumption \ref{a1} and \ref{a3},
the average of local update after the clipping operation from selected clients is 
\begin{equation}
    \mathbb{E}\|\frac{1}{m}\sum_{i\in \mathcal{W}^t}\tilde{\Delta}_i^t\|^2 \leq 3K\eta^2(L^2\rho^2+B^2) \nonumber
\end{equation}
\end{lemma}
\begin{proof}
\begin{equation}
\begin{split}
    \mathbb{E}\|\frac{1}{m}\sum_{i\in \mathcal{W}^t}\tilde{\Delta}_i^t\|^2 
    & \leq \mathbb{E} \|\frac{1}{m}\sum_{i\in \mathcal{W}^t}\sum_{i=0}^{K-1}\eta \tilde{\mathbf{g}}^{t,k}(i) \cdot \alpha_i^t\|^2  \leq \frac{\eta^2}{m} \sum_{i\in \mathcal{W}^t}\sum_{i=0}^{K-1}\mathbb{E} \| \nabla F_i(\mathbf{w}^{t,k}(i) +\delta; \xi_i) - \nabla F_i(\mathbf{w}^{t,K}(i); \xi_i)\\
    & + \nabla F_i(\mathbf{w}^{t,k}(i); \xi_i) - \nabla F_i(\mathbf{w}^{t}(i)) + \nabla F_i(\mathbf{w}^{t}(i))\|^2\\
    & \overset{a)}{\leq} 3K\eta^2(L^2\rho^2+B^2) ,\nonumber
\end{split}
\end{equation}
where a) uses assumption \ref{a1} and \ref{a3} and 
\begin{equation}
    \alpha_i^t := \min \Big (1, \frac{C}{\eta \|\sum_{k=0}^{K-1}\tilde{\mathbf{g}}^{t,k}(i)\|_2} \Big).\nonumber
\end{equation}
\end{proof}


\subsection{Proof of Sensitivity Analysis}
\begin{proof}[Proof of Theorem \ref{th:sensitivity}]
Recall that the local update before clipping and adding noise on client $i$ is $\Delta_i^t = \mathbf{w}^{t,K}(i)-\mathbf{w}^{t,0}(i)$. Then,
\begin{equation}
 \begin{split}
    \mathbb{E} \mathcal{S}_{\Delta_i^t}^2 & = \max  \mathbb{E} \| \Delta_i^t(\mathbf{x}) - \Delta_i^t(\mathbf{y})\|_2^2\\
    & = \mathbb{E} \| \mathbf{x}^{t,K}(i) - \mathbf{x}^{t,0}(i) - (\mathbf{y}^{t,K}(i) - \mathbf{y}^{t,0}(i))\|_2^2\\
    & =  \eta^2\mathbb{E} \sum_{k=0}^{K-1}\|\nabla F_i(\mathbf{x}^{t,k}(i)+\delta_x; \xi_i) -\nabla F_i(\mathbf{y}^{t,k}(i)+\delta_y; \xi_i^{'})\|_2^2 \\
    & =  \eta^2L^2 \mathbb{E} \sum_{k=0}^{K-1}\|\mathbf{y}^{t,k}(i)- \mathbf{x}^{t,k}(i) + (\delta_y-\delta_x)\|_2^2  \\
    & \overset{a)}{\leq} 2\eta^2L^2 K \max \|\Delta_i^t(\mathbf{y})-\Delta_i^t(\mathbf{x})\|_2^2 
    + 2\eta^2L^2\rho^2 \mathbb{E} \sum_{k=0}^{K-1} 
    \Big \|\frac{\nabla
    F_i(\mathbf{y}^{t,k}(i)+\delta_y; \xi_i^{'})}{\|\nabla
    F_i(\mathbf{y}^{t,k}(i)+\delta_y; \xi_i^{'})\|_2}- \frac{\nabla
    F_i(\mathbf{y}^{t,k}(i); \xi_i^{'})}{\|\nabla
    F_i(\mathbf{y}^{t,k}(i); \xi_i^{'})\|_2}\\
    & + ( \frac{\nabla
    F_i(\mathbf{x}^{t,k}(i); \xi_i)}{\|\nabla
    F_i(\mathbf{x}^{t}(i,k); \xi_i)\|_2} - \frac{\nabla
    F_i(\mathbf{x}^{t,k}(i)+\delta_x; \xi_i)}{\|\nabla
    F_i(\mathbf{x}^{t,k}(i)+\delta_x; \xi_i)\|_2})
     + \frac{\nabla
    F_i(\mathbf{y}^{t,k}(i); \xi_i^{'}))}{\|\nabla
    F_i(\mathbf{y}^{t,k}(i);\xi_i^{'}))\|_2} - \frac{\nabla
    F_i(\mathbf{x}^{t,k}(i);\xi_i)}{\|\nabla
    F_i(\mathbf{x}^{t,k}(i); \xi_i)\|_2}\Big \|_2^2\\
    & \leq 2\eta^2L^2K \max \|\Delta_i^t(\mathbf{y})-\Delta_i^t(\mathbf{x})\|_2^2 +
     6\eta^2\rho^2L^2 \mathbb{E} \sum_{k=0}^{K-1}\Big(4  + \frac{1}{\rho^2}\Big\|\rho \frac{\nabla
    F_i(\mathbf{y}^{t,k}(i); \xi_i^{'}))}{\|\nabla
    F_i(\mathbf{y}^{t,k}(i);\xi_i^{'}))\|_2} - \rho \frac{\nabla
    f(\mathbf{y}^{t})}{\|\nabla
    f(\mathbf{y}^{t})\|_2} \\
    &+ (  \rho \frac{\nabla
    f(\mathbf{x}^{t})}{\|\nabla
    f(\mathbf{x}^{t})\|_2} - \rho \frac{\nabla
    F_i(\mathbf{x}^{t,k}(i);\xi_i)}{\|\nabla
    F_i(\mathbf{x}^{t,k}(i); \xi_i)\|_2}) 
    + \rho \frac{\nabla
    f(\mathbf{y}^{t})}{\|\nabla
    f(\mathbf{y}^{t})\|_2} - \rho \frac{\nabla
    f(\mathbf{x}^{t})}{\|\nabla
    f(\mathbf{x}^{t})\|_2}\Big\|_2^2 \Big )\\
    & \overset{b)}{\leq}
    2\eta^2L^2 K \mathcal{S}_{\Delta_i^t}^2 + 6\eta^2\rho^2KL^2(4  + 12K^2L^2\eta^2+ 6) \\
    & \leq \frac{6\eta^2\rho^2KL^2(12K^2L^2\eta^2+ 10)}{1-2\eta^2L^2 K}
\end{split}   
\end{equation}
where a) and b)  uses lemma \ref{y_k-x_k} and \ref{e_delta}, respectively. 

When the local adaptive learning rate satisfies $\eta=\mathcal{O}({1}/{L\sqrt{KT}})$ and the perturbation amplitude $\rho$
proportional to the learning rate, e.g., $\rho = \mathcal{O}(\frac{1}{\sqrt{T}})$, we have
\begin{align}
\small
    \mathbb{E}\mathcal{S}^2_{\Delta_i^t} \leq
    \mathcal{O}\left(\frac{1}{T^2}\right). 
\end{align}
\end{proof}

For comparison, we also present the expected squared sensitivity of local update with SGD in DPFL as follows. It is clearly seen that the upper bound in $  \mathbb{E}\mathcal{S}^2_{\Delta_i^t, SAM}$ is tighter than that in $\mathbb{E}\mathcal{S}^2_{\Delta_i^t, SGD}$.


\begin{proof}[Proof of sensitivity with SGD in FL.]
\begin{equation}
    \begin{split}
         \mathbb{E}  \mathcal{S}_{\Delta_i^t, SGD}^2 & = \max  \mathbb{E}  \| \Delta_i^t(\mathbf{x}) - \Delta_i^t(\mathbf{y})\|_2^2
     = \eta^2 \mathbb{E} \sum_{i=0}^{K-1}\|\nabla F_i( \mathbf{x}^{t,k}(i); \xi_i) - \nabla F_i( \mathbf{y}^{t,k}(i);\xi_i^{'})\|_2^2\\
    & = \eta^2 \mathbb{E} \sum_{i=0}^{K-1}\|\nabla F_i( \mathbf{x}^{t,k}(i); \xi_i) -\nabla F_i( \mathbf{x}^{t}(i)) + \nabla F_i( \mathbf{x}^{t}(i)) 
    - \nabla F_i( \mathbf{y}^{t}(i))  +\nabla F_i( \mathbf{y}^{t}(i)) -
    \nabla F_i( \mathbf{y}^{t,k}(i);\xi_i^{'})\|_2^2\\
    & \overset{a)}{\leq}
    3\eta^2 \mathbb{E} \sum_{i=0}^{K-1}(2\sigma_l^2+L^2\|y^{t,k}(i)-x^{t,k}(i)\|_2^2)\\
    & \overset{b)}{\leq} 6\eta^2K\sigma_l^2+3\eta^2L^2K  \max  \mathbb{E} \| \Delta_i^t(\mathbf{x}) - \Delta_i^t(\mathbf{y})\|_2^2\\
    & \leq \frac{6\eta^2\sigma_l^2K}{1-3\eta^2KL^2}.
    \end{split}
\end{equation}
Where a) and b) uses assumptions \ref{a1}-\ref{a2} and lemma \ref{y_k-x_k}, respectively. Thus $\mathbb{E}\mathcal{S}^2_{\Delta_i^t, SGD} \leq \mathcal{O}(\frac{\sigma_l^2}{KL^2T})$ when $\eta=\mathcal{O}({1}/{L\sqrt{KT}})$.
\end{proof}

\subsection{Proof of Convergence Analysis}
\begin{proof}[Proof of Theorem \ref{th:conver}]
We define the following notations for convenience:
\begin{equation}
    \begin{split}
        & \tilde{\Delta}_i^t = -\eta\sum_{k=0}^{K-1}\tilde{\mathbf{g}}^{t,k}(i) \cdot \alpha_i^t;\\
        & \overline{\Delta_i^t} = -\eta\sum_{k=0}^{K-1}\tilde{\mathbf{g}}^{t,k}(i) \cdot \overline{\alpha}^t, \nonumber
    \end{split}
\end{equation}
where 
\begin{equation}
    \begin{split}
        & \alpha_i^t := \min \Big (1, \frac{C}{\eta \|\sum_{k=0}^{K-1}\tilde{\mathbf{g}}^{t,k}(i)\|} \Big), \\
        & \overline{\alpha}^t := \frac{1}{M}\sum_{i=1}^{M} \alpha_i^t,\\
        & \tilde{\alpha}^t :=\frac{1}{M}\sum_{i=1}^{M} |\alpha_i^t - \overline{\alpha}^t|.
    \end{split} \nonumber
\end{equation}
The Lipschitz continuity of $\nabla f$:
\begin{equation}
    \begin{split}
        & \mathbb{E} f(\mathbf{w}^{t+1}) \\
        & \leq \mathbb{E} f(\mathbf{w}^t) + \mathbb{E} \Big \langle\nabla f(\mathbf{w}^{t}), \mathbf{w}^{t+1}-\mathbf{w}^t \Big \rangle
        + \mathbb{E} \frac{L}{2}\|\mathbf{w}^{t+1}-\mathbf{w}^t\|^2\\
        & = \mathbb{E} f(\mathbf{w}^t) + \mathbb{E}\Big  \langle \nabla f(\mathbf{w}^{t}), \frac{1}{m}\sum_{i\in \mathcal{W}^t}\tilde{\Delta}_i^t + z_i^t \Big \rangle
         + \frac{L}{2}\mathbb{E}\Big  \|\frac{1}{m}\sum_{i\in \mathcal{W}^t}\tilde{\Delta}_i^t + z_i^t \Big \|^2\\
        & = \mathbb{E} f(\mathbf{w}^t) + 
         \underbrace{\Big \langle \nabla f(\mathbf{w}^{t}), \mathbb{E} \frac{1}{m}\sum_{i\in \mathcal{W}^t}\tilde{\Delta}_i^t \Big \rangle}_{\text{I}}
        + 
        \frac{L}{2} \mathbb{E} \underbrace{\Big \langle \|\frac{1}{m}\sum_{i\in \mathcal{W}^t}\tilde{\Delta}_i^t\|^2\Big \rangle}_{\text{II}} + \frac{L\sigma^2C^2pd}{2m^2} ,
    \end{split}
\end{equation}
where $d$ represents dimension of $\mathbf{w}_i^{t,k}$, $p$ is the sparsity ratio, and the mean of noise $z_i^t$ is zero. Then, we analyze I and II, respectively.\\
For I, we have
\begin{equation}
    \begin{split}
        & \Big \langle\nabla f(\mathbf{w}^{t}), \mathbb{E} \frac{1}{m}\sum_{i\in \mathcal{W}^t}\tilde{\Delta}_i^t \Big  \rangle = \Big \langle\nabla f(\mathbf{w}^{t}), \mathbb{E}\frac{1}{M}\sum_{i=1}^M\tilde{\Delta}_i^t-\overline{\Delta}_i^t  \Big \rangle
        + \Big \langle\nabla f(\mathbf{w}^{t}), \mathbb{E}\frac{1}{M}\sum_{i=1}^M \overline{\Delta}_i^t  \Big \rangle.
    \end{split}
\end{equation}
Then we bound the two terms in the above equality, respectively. For the first term, we have
\begin{equation}
    \begin{split}
        &\mathbb{E} \Big \langle\nabla f(\mathbf{w}^{t}), \mathbb{E}\frac{1}{M}\sum_{i=1}^M\tilde{\Delta}_i^t-\overline{\Delta}_i^t  \Big \rangle\\
        & \leq\mathbb{E} \Big \langle\nabla f(\mathbf{w}^{t}), \mathbb{E}\frac{1}{M}\sum_{i=1}^M \sum_{k=0}^{K-1}\eta |\alpha_i^t - \overline{\alpha}^t|\tilde{\mathbf{g}}^{t,k}(i)\Big \rangle\\
        & \leq \frac{\eta K}{M}\sum_{i=1}^{M}\mathbb{E}|\alpha_i^t - \overline{\alpha}^t| \Big \langle \nabla F_i(\mathbf{w}^{t}),\tilde{\mathbf{g}}^{t,k}(i) \Big \rangle\\
        & \overset{a)}{\leq} \frac{\eta K}{M}\sum_{i=1}^{M}\mathbb{E}|\alpha_i^t - \overline{\alpha}^t| \Big(-\frac{1}{2}(\|\nabla F_i(\mathbf{w}^{t,k})\|^2+\|F_i(\mathbf{w}^{t,k}+\delta; \xi_i)\|^2)
        + \frac{1}{2} \|\nabla F_i(\mathbf{w}^{t,k}+\delta; \xi_i) - \nabla F_i(\mathbf{w}^{t,k}; \xi_i)\|^2\Big)\\
        & \overset{b)}{\leq}\eta \tilde{\alpha}^t K(\frac{1}{2}L^2\rho^2-B^2),
    \end{split}
\end{equation}
where $\tilde{\alpha}^t =\frac{1}{M}\sum_{i=1}^{M} |\alpha_i^t - \overline{\alpha}^t|$, a) uses $\langle a,b \rangle = -\frac{1}{2}\|a\|^2-\frac{1}{2}\|b\|^2 + \frac{1}{2}\|a - b\|^2$ and b) bases on assumption \ref{a1},\ref{a3}.\\
For the second term, we have
\begin{equation}
    \begin{split}
         & \Big \langle\nabla f(\mathbf{w}^{t}), \mathbb{E}\frac{1}{M}\sum_{i=1}^M \overline{\Delta}_i^t  \Big \rangle\\
         & \overset{a)}{\leq} \frac{- \overline{\alpha}^t\eta K}{2}\|\nabla f(\mathbf{w}^{t})\|^2 - \frac{\overline{\alpha}^t}{2K} \mathbb{E}\Big \|\frac{1}{\overline{\alpha}^t M}\sum_{i=1}^M \overline{\Delta}_i^t \Big \|^2
          + \frac{ \overline{\alpha}^t}{2}
         \underbrace{\mathbb{E}\Big \|\sqrt{K}\nabla f(\mathbf{w}^{t})- \frac{1}{ \overline{\alpha}^tM\sqrt{K}}\sum_{i=1}^M \overline{\Delta}_i^t \Big \|^2}_{\text{III}},
    \end{split}
\end{equation}
where a) uses $\langle a,b \rangle = -\frac{1}{2}\|a\|^2-\frac{1}{2}\|b\|^2 + \frac{1}{2}\|a - b\|^2$ and $0< \eta <1 $. Next, we bound III as follows:
\begin{equation}
    \begin{split}
         \text{III} &= K\mathbb{E}\Big \| \nabla f(\mathbf{w}^{t}) + \frac{1}{MK}\sum_{i=1}^M\sum_{k=0}^{K-1} \nabla \eta F_i(\mathbf{w}^{t,k}+\delta; \xi_i) \Big\|^2\\
        & \leq \frac{1}{M}\sum_{i=1}^M\sum_{k=0}^{K-1} \mathbb{E}\Big \| \eta (F_i(\mathbf{w}^{t,k}+\delta; \xi_i) - \nabla F_i(\mathbf{w}^{t,k}; \xi_i)) 
        + \eta (\nabla F_i(\mathbf{w}^{t,k}; \xi_i) - \nabla F_i(\mathbf{w}^{t})) + (1+\eta) \nabla F_i(\mathbf{w}^{t})\Big \| ^2\\
        & \overset{a)}{\leq} 3K\eta^2L^2 \Big( \rho^2 + \mathbb{E} \|\mathbf{w}^{t,k} - \mathbf{w}^{t}\|^2 + 2B^2 \Big)\\
        & \overset{b)}{\leq} 3K\eta^2L^2 \Big[ \rho^2 +  5K\eta^2 \Big(2L^2 \rho^2 \sigma_l^2+ 6K(3\sigma_g^2 + 6L^2 \rho^2 ) 
        + 6K\|\nabla f(\mathbf{w}^t)\|^2 \Big) + 24K^3 \eta^4 L^4 \rho^2 + B^2\Big],
    \end{split}
\end{equation}
where $0< \eta <1$, a) and b uses assumption \ref{a1}, \ref{a3} and lemma \ref{e_w}, respectively.\\
For II, we uses lemma \ref{Delta_average}. Then, combining Eq. 12-16, we have
\begin{equation}
    \begin{split}
         \mathbb{E} f(\mathbf{w}^{t+1}) 
        & \leq \mathbb{E} f(\mathbf{w}^t) + \eta \tilde{\alpha}_t K(\frac{1}{2}L^2\rho^2-B^2)
        - \frac{\overline{\alpha}^t\eta K}{2}\|\nabla f(\mathbf{w}^{t})\|^2 -
         \frac{\eta \overline{\alpha}^t}{2K} \mathbb{E}\Big \|\frac{1}{\eta\overline{\alpha}^t M}\sum_{i=1}^M \overline{\Delta}_i^t \Big \|^2 \\
        & + \frac{3\overline{\alpha}^t\eta^2L^2K}{2}\Big[ \rho^2 +  5K\eta^2 \Big(2L^2 \rho^2 \sigma_l^2+ 6K(3\sigma_g^2 + 6L^2 \rho^2 )  + 6K\|\nabla f(\mathbf{w}^t)\|^2 \Big)\\
        &  + 24K^3 \eta^4 L^4 \rho^2 + B^2\Big ]
         + \frac{3\eta^2KL(L^2\rho^2+B^2)}{2} + \frac{L\sigma^2C^2pd}{2m^2}.
    \end{split}
\end{equation}
When $\eta \leq \frac{1}{3\sqrt{KL}}$, the inequality is 
\begin{equation}
    \begin{split}
         \mathbb{E} f(\mathbf{w}^{t+1}) 
         & \leq \mathbb{E} f(\mathbf{w}^t) - \frac{ \overline{\alpha}^t \eta K}{2}\mathbb{E} \|\nabla f(\mathbf{w}^t)\|^2 + \frac{\tilde{\alpha}^t \eta KL^2\rho^2}{2}  + \frac{3\overline{\alpha}^t\eta^2 KL^2\rho^2}{2}
        -\tilde{\alpha}^t \eta KB^2 \\
        & + \frac{15\overline{\alpha}^t K\eta^4L^2}{2} \Big(2L^2 \rho^2 \sigma_l^2 + 6K(3\sigma_g^2 + 6L^2 \rho^2 )  + 6K\|\nabla f(\mathbf{w}^t)\|^2 \Big) + 36\eta^6K^4L^6\rho^2 \\
        & + \frac{3\eta^2KL(L^2 \rho^2+B^2)}{2} + \frac{L\sigma^2C^2pd}{2m^2}.
    \end{split}
\end{equation}
Sum over $t$ from $1$ to $T$, we have
\begin{equation}
    \begin{split}
         \frac{1}{T}\sum_{t=1}^T \mathbb{E} \Big[\overline{\alpha}^t\|f(\mathbf{w}^t)\|^2\Big] & \leq \frac{2L(f({\bf w}^{1})-f^{*})}{\sqrt{KT}} + \frac{1}{T}\sum_{t=1}^T \tilde{\alpha}^t  L^2\rho^2  - 2 \tilde{\alpha}^t  B^2 + 30\eta^2 L^2\frac{1}{T}\sum_{t=1}^T \overline{\alpha}^t \Big(2L^2 \rho^2 \sigma_l^2 + 6K(3\sigma_g^2 + 6L^2 \rho^2 ) \Big)\\
        & + 72\eta^4K^3L^6\rho^2+3\eta L(L^2 \rho^2+B^2)  + \frac{L\sigma^2C^2pd}{\eta m^2K}
    \end{split}
\end{equation}
Assume the local adaptive learning rate satisfies $\eta=\mathcal{O}({1}/{L\sqrt{KT}})$, both $\frac{1}{T}\sum_{t=1}^T \tilde{\alpha}^t $ and $\frac{1}{T}\sum_{t=1}^T \overline{\alpha}^t $ are two important parameters for measuring the impact of clipping. Meanwhile, both $\frac{1}{T}\sum_{t=1}^T \tilde{\alpha}^t $ and $\frac{1}{T}\sum_{t=1}^T \overline{\alpha}^t $ are also bounded by $1$. Then, our result is
\begin{equation}
    \begin{split}
            & \frac{1}{T} \sum_{t=1}^T
    \mathbb{E}\left[\overline{\alpha}^{t}\left\|\nabla f\left(\mathbf{w}^{t}\right)\right\|^{2}\right]   \leq  
    \underbrace{\mathcal{O}\left(\frac{2L(f({\bf w}^{1})-f^{*})}{\sqrt{KT}} + \frac{\sigma_{l}^2 L^2\rho}{KT}\right)}_{\text{From FedSAM}}   +
    \underbrace{
     \underbrace{\mathcal{O}\left(\sum_{t=1}^T( \frac{\overline{\alpha}^t  \sigma_{g}^2 }{T^2} + \frac{\tilde{\alpha}^t  L^2\rho^2 }{T} ) \right)}_{\text{Clipping}}
    + \underbrace{ \mathcal{O}\left(\frac{L^2 \sqrt{T}\sigma^2C^2pd}{m^2\sqrt{K}} \right)}_{\text{Adding noise}}
    }_{\text{From operations for DP}} . 
    \end{split}
\end{equation}
Assume the perturbation amplitude $\rho$
proportional to the learning rate, e.g., $\rho = \mathcal{O}(\frac{1}{\sqrt{T}})$, we have
\begin{equation}
    \begin{split}
            & \frac{1}{T} \sum_{t=1}^T
    \mathbb{E}\left[\overline{\alpha}^{t}\left\|\nabla f\left(\mathbf{w}^{t}\right)\right\|^{2}\right]   \leq  
    \underbrace{\mathcal{O}\left(\frac{2L(f({\bf w}^{1})-f^{*})}{\sqrt{KT}} + \frac{ L^2\sigma_{l}^2}{KT^2}\right)}_{\text{From FedSAM}}   +
    \underbrace{
     \underbrace{\mathcal{O}\left(\sum_{t=1}^T( \frac{\overline{\alpha}^t  \sigma_{g}^2 }{T^2} + \frac{\tilde{\alpha}^t  L^2 }{T^2} ) \right)}_{\text{Clipping}}
    + \underbrace{ \mathcal{O}\left(\frac{L^2 \sqrt{T}\sigma^2C^2pd}{m^2\sqrt{K}} \right)}_{\text{Adding noise}}
    }_{\text{From operations for DP}} . 
    \end{split}
\end{equation}
\end{proof}

\subsection{Proof of Generalization Analysis} \label{ap:gener}

\begin{proof}
The main proof can be seen as acquiring generalization bound through the lens of differential privacy (DP). The proof skeleton can be concluded in three stages: (1) We first take a global model of the proposed algorithms and thus classify it as an iterative machine learning algorithm. (2) We then calculate the differential privacy of each step/round in the algorithm. (3) Extend it to Differentially Private Federated Learning through the bridges provided in Section 5 of \cite{he2021tighter}.

First, a global model $\mathbf{w}^t$ always exists during the overall training process. 
For simplicity, we define $\mathbf{w}^t$ as the global model at iteration $t$, which is also the same as the communication round. Then FL paradigm can be seen as iteratively optimizing $\mathbf{w}^t$ using gradient information on $m$ participated clients. We also denote $\mathcal{N}(0, \sigma^2C^2 \cdot \mathbf{I}_d/m)$ as the added Gaussian noise in DP, where $\sigma^2$ is the Gaussian noise variance. We define $\tau$ as $m$ participated clients and overall iteration steps as $T$.  The diameter of the local update model space is defined as $D\overset{\triangle}{=}\max_{\mathbf{w},z,z'}\Vert \nabla\ell(z,\mathbf{w})- \nabla\ell(z',\mathbf{w})\Vert =  \max_{\mathbf{w},z,z'} \| \max_{\|\delta\|_2 \leq \rho}(\nabla \ell(\mathbf{w} +\delta_z; z) - \nabla \ell(\mathbf{w} +\delta_{z'}; z'))\| $ in our algorithms by using SAM local optimizer. We also denote $G_{\mathcal{B}}(\mathbf{w})\overset{\triangle}{=}\frac{1}{\Vert \mathcal{B}\Vert}\sum_{z\in \mathcal{B}}g(z,\mathbf{w})$ as the mean of $g$ over $\mathcal{B}$ for brevity. We also use $\boldsymbol{p}$ as the probability density, with $\boldsymbol{p}^{V}$ the probability density conditional on any random variable $V$.

Afterward, due to the local update being clipped, we have the diameter of the local update model space defined as \cite{he2021tighter}
\begin{equation}
    \begin{split}
         D & = \max_{\mathbf{w},z,z'} \| \max_{\|\delta\|_2 \leq \rho}(\nabla \ell(\mathbf{w} +\delta_z; z) - \nabla \ell(\mathbf{w} +\delta_{z'}; z'))\| \\
    & \overset{a}{=} 2L\rho, \nonumber
    \end{split}
\end{equation}

where a) uses assumption 1 and $\delta=\rho\frac{ g(z,\mathbf{w})}{\left \| g(z,\mathbf{w}) \right \|_2}$.
Then we calculate the differential privacy of each step/round. Recall Algorithm \ref{DFedSAM_DP}, line 5-10 denotes the local training process, the gradient information. Line 3 in Algorithm \ref{DFedSAM_DP} is equivalent to uniformly sampling a mini-batch $\mathcal{I}_t$ from index set $[N]$ with size $\tau$ without replacement and letting $\mathcal{B}_t=S_{\mathcal{I}_t}$. Furthermore, for fixed $\mathbf{w}^{t-1}$, $\mathcal{I}$, and any two adjacent sample sets $S$ and $S'$, we have
\begin{equation}
	\begin{aligned}
    \frac{\boldsymbol{p}^{S,\mathcal{I}_t}(\mathbf{w}^{t}=\mathbf{w}\vert \mathbf{w}^{t-1})}{\boldsymbol{p}^{S',\mathcal{I}_t}(\mathbf{w}^t=\mathbf{w}\vert \mathbf{w}^{t-1})}
	&=\frac{\boldsymbol{p}^{S,\mathcal{I}_t}(\eta_t(G_{S_{\mathcal{I}}}(\mathbf{w}^{t-1})+\mathcal{N}(0, \sigma^2C^2 \cdot \mathbf{I}_d/m))=\mathbf{w}-\mathbf{w}^{t-1})}{\boldsymbol{p}^{S',\mathcal{I}_t}(\eta_t(G_{S'_{\mathcal{I}}}(\mathbf{w}^{t-1})+\mathcal{N}(0, \sigma^2C^2 \cdot \mathbf{I}_d/m))=\mathbf{w}-\mathbf{w}^{t-1})}
	\\
	&=\frac{\boldsymbol{p}^{\mathcal{I}_t,\mathbf{w}^{t-1}}(\mathcal{N}(0, \sigma^2C^2 \cdot \mathbf{I}_d/m)=\mathbf{w}')}{\boldsymbol{p}^{S,S',\mathcal{I}_t,\mathbf{w}^{t-1}}(G_{S'_{\mathcal{I}}}(\mathbf{w}^{t-1})-G_{S_{\mathcal{I}}}(\mathbf{w}^{t-1})+\mathcal{N}(0, \sigma^2C^2 \cdot \mathbf{I}_d/m)=\mathbf{w}')},
	\end{aligned}
\end{equation}
where $\eta_t \mathbf{w}'=\mathbf{w}-\mathbf{w}^{t-1}-\eta_t G_{S_{\mathcal{I}}}(\mathbf{w}^{t-1})$. Therefore, when consider the additive Gaussian noise into consideration, if $\mathbf{w}\sim \mathbf{w}^{t-1}+\eta_t (G_{S_{\mathcal{I}}}(\mathbf{w}^{t-1})+\mathcal{N}(0, \sigma^2C^2 \cdot \mathbf{I}_d/m))$, then $\mathbf{w}'\sim G_{S_{\mathcal{I}}}(\mathbf{w}^{t-1})+\mathcal{N}(0, \sigma^2C^2 \cdot \mathbf{I}_d/m) $.

For simplicity, according to the definition of differential privacy, we define 
	\begin{equation}
	D_p^{S,S',\mathcal{I}_t,\mathbf{w}^{t-1}}(\mathbf{w}')=\log \frac{\boldsymbol{p}^{\mathcal{I}_t,\mathbf{w}^{t-1}}(\mathcal{N}(0, \sigma^2C^2 \cdot \mathbf{I}_d/m)=\mathbf{w}')}{\boldsymbol{p}^{S,S',\mathcal{I}_t,\mathbf{w}^{t-1}}(G_{S'_{\mathcal{I}}}(\mathbf{w}^{t-1})-G_{S_{\mathcal{I}}}(\mathbf{w}^{t-1})+\mathcal{N}(0, \sigma^2C^2 \cdot \mathbf{I}_d/m)=\mathbf{w}')},
	\end{equation}
which by the definition of Gaussian distribution further leads to
\begin{equation}
    \begin{aligned}
	D_p(\mathbf{w}')
	=&-\frac{\Vert \mathbf{w}'\Vert^2}{2\sigma^2C^2 d^2/m^2}+\frac{\Vert \mathbf{w}'-G_{S'_{\mathcal{I}}}(\mathbf{w}^{t-1})+G_{S_{\mathcal{I}}}(\mathbf{w}^{t-1})\Vert^2}{2\sigma^2C^2 d^2/m^2}
	\\
	=&\frac{2\langle \mathbf{w}',-G_{S'_{\mathcal{I}}}(\mathbf{w}^{t-1})+G_{S_{\mathcal{I}}}(\mathbf{w}^{t-1})\rangle+\Vert G_{S'_{\mathcal{I}}}(\mathbf{w}^{t-1})-G_{S_{\mathcal{I}}}(\mathbf{w}^{t-1})\Vert^2}{2\sigma^2C^2 d^2/m^2}.
	\end{aligned}
\end{equation}

Denote $-G_{S'_{\mathcal{I}}}(\mathbf{w}^{t-1})+G_{S_{\mathcal{I}}}(\mathbf{w}^{t-1})$ as $\mathbf{v}$. By the definition of $D=2L\rho$ (the diameter of the gradient space), we have
\begin{equation}
	\Vert \mathbf{v} \Vert<\frac{1}{m} D < \frac{2L\rho}{m}.
\end{equation}

On the other hand, since $\langle \mathbf{v}, \mathbf{w}'\rangle\sim \mathcal{N}(0,\Vert \mathbf{v}\Vert^2\sigma^2C^2 d^2/m^2)$,
by Chernoff Bound technique, we have
\begin{equation}
   	\begin{aligned}
	\mathbb P\left(\langle \mathbf{v}, \mathbf{w}'\rangle\ge \frac{2\sqrt{2}L\rho \sigma Cd}{m^2}\sqrt{\log\frac{1}{\tilde\delta}}\right)& \le	\mathbb P\left(\langle \mathbf{v}, \mathbf{w}'\rangle\ge \frac{\sqrt{2}\Vert \mathbf{v}\Vert\sigma Cd }{m}\sqrt{\log\frac{1}{\tilde\delta}}\right)\\
	&\le \min_{1\leq t \leq T}\exp\left(-\frac{\sqrt{2}t\Vert \mathbf{v}\Vert\sigma Cd }{m} \sqrt{\log\frac{1}{\tilde\delta}}\right)\mathbb{E}(e^{t\langle \mathbf{v}, \mathbf{w}'\rangle}).
	\end{aligned} 
\end{equation}

Where $0<\tilde \delta\leq1$ is an arbitrary positive real constant, and then we define 
\begin{equation} \label{delta}
    \delta =  \min_{1\leq t \leq T}\exp\left(-\frac{\sqrt{2}t\Vert \mathbf{v}\Vert\sigma Cd }{m} \sqrt{\log\frac{1}{\tilde\delta}}\right)\mathbb{E}(e^{t\langle \mathbf{v}, \mathbf{w}'\rangle}).
\end{equation}
Therefore, with probability at least $1-\delta$ with respect to $\mathbf{w}'$, we have that 
\begin{equation}
     \begin{aligned}
         D_p(\mathbf{w}') & \le \frac{\sqrt{2}L\rho \sigma C d\sqrt{\log\frac{1}{\tilde \delta}}+2L^2\rho^2}{\sigma^2 C^2d^2}\\
         & \le \frac{L\rho \sqrt{2\log\frac{1}{\tilde \delta}}}{\sigma C d} + \left(\frac{\sqrt{2}L\rho}{\sigma Cd}\right)^2.
     \end{aligned}
\end{equation}
Combining Lemma \ref{lemma:one}, we can have that the each step in Algorithm \ref{DFedSAM_DP} is ($\tilde{\varepsilon},\frac{m}{N}\delta$)-differentially private, where $\tilde\varepsilon$ is defined as
\begin{equation}
	\tilde\varepsilon =\log \left(\frac{N-m}{N}+\frac{m}{N}\exp\left(\frac{L\rho \sqrt{2\log\frac{1}{\tilde \delta}}}{\sigma C d} + \left(\frac{\sqrt{2}L\rho}{\sigma Cd}\right)^2\right)\right).
\end{equation}

Applying Lemma \ref{lemma:multi}, we can conclude the differentially private guarantee ($\varepsilon',\delta'$) for the iterative steps.

Finally, combining Lemma \ref{lemma:g-dp} with ($\varepsilon',\delta'$) finish the proof.
\end{proof}

\end{document}